\documentclass[11pt]{article}
\usepackage{times}
\usepackage{fullpage,url,graphicx}

\usepackage{amsmath}
\usepackage{amssymb, latexsym, amsfonts, amscd, amsthm, mathrsfs, enumerate, esint}

\newtheorem{observation}{Observation}

\usepackage{framed}
\usepackage{algorithm}
\newtheorem{lemma}{Lemma}
\newtheorem{proposition}{Proposition}
\newtheorem{corollary}{Corollary}
\newtheorem{theorem}{Theorem}
\newtheorem{definition}{Definition}

\newcommand{\E}{\mathbb{E}}
\newcommand{\N}{\mathbb{N}}
\newcommand{\R}{\mathbb{R}}

\DeclareMathOperator{\polylog}{polylog}

\DeclareMathOperator{\dPois}{Pois}
\DeclareMathOperator{\Var}{Var}
\DeclareMathOperator{\eps}{\varepsilon}
\DeclareMathOperator{\dBin}{Bin}
\DeclareMathOperator{\dUnif}{Unif}
 
 \author{Bhaswar B. Bhattacharya\\
Department of Statistics, Stanford University, California, USA \\
bhaswar@stanford.edu
\and
Gregory Valiant\thanks{This work is supported in part by NSF CAREER Award CCF-1351108.}\\
Department of Computer Science, Stanford University, California, USA\\
valiant@stanford.edu
}
\title{Testing Closeness With Unequal Sized Samples}

\begin{document}

\maketitle

\begin{abstract}
We consider the problem of closeness testing for two discrete distributions in the practically relevant setting of \emph{unequal} sized samples drawn from each of them.  Specifically, given a target error parameter $\eps > 0$,  $m_1$ independent draws from an unknown distribution $p,$ and $m_2$ draws from an unknown distribution $q$, we describe a test for distinguishing the case that $p=q$ from the case that $||p-q||_1 \geq \eps$. If $p$ and $q$ are supported on at most $n$ elements, then our test is successful with high probability provided $m_1\geq n^{2/3}/\varepsilon^{4/3}$ and $m_2 = \Omega\left(\max\{\frac{n}{\sqrt m_1\varepsilon^2}, \frac{\sqrt n}{\varepsilon^2}\}\right);$ we show that this tradeoff is optimal throughout this range, to constant factors.   These results extend the recent work of Chan et al. \cite{chan_valiant}  who established the sample complexity when the two samples have equal sizes, and tightens the results of Acharya et al. \cite{orlitsky_unequal} by polynomials factors in both $n$ and $\eps$.  As a consequence, we obtain an algorithm for estimating the mixing time of a Markov chain on $n$ states up to a $\log n$ factor that uses $\tilde{O}(n^{3/2} \tau_{mix})$ queries to a ``next node'' oracle, improving upon the $\tilde{O}(n^{5/3}\tau_{mix})$ query algorithm of~\cite{batu}.   Finally, we note that the core of our testing algorithm is a relatively simple statistic that seems to perform well in practice, both on synthetic data and on natural language data.
\end{abstract}
%
%\begin{keywords}
%List of keywords
%\end{keywords}

\section{Introduction}

One of the most fundamental problems in statistical hypothesis testing is the question of distinguishing whether two unknown distributions are very similar, or significantly different.  Classical tests, like the Chi-squared test or the Kolmogorov-Smirnov statistic, are optimal in the asymptotic regime, for fixed distributions as the sample sizes tend towards infinity.  Nevertheless, in many modern settings---such as the analysis of customer data, web logs, natural language processing, and genomics, despite the quantity of available data---the support sizes and complexity of the underlying distributions are far larger than the datasets, as evidenced by the fact that many phenomena are observed only a single time in the datasets, and the empirical distributions of the samples are poor representations of the true underlying distributions.\footnote{To give some specific examples, two recent independent studies~\cite{sci2,sci3} each considered the genetic sequences of over 14,000 individuals, and found that rare variants are extremely abundant, with over 80\% of mutations observed just once in the sample.  A separate recent paper~\cite{sci1} found that the discrepancy in rare mutation abundance cited in different demographic modeling studies can largely be explained by discrepancies in the sample sizes of the respective studies, as opposed to differences in the actual distributions of rare mutations across demographics, highlighting the importance of improved statistical tests in this ``undersampled'' regime.}  In such settings, we must understand these statistical tasks not only in the asymptotic regime (in which the amount of available data goes to infinity), but in the ``undersampled'' regime in which the dataset is significantly smaller than the size or complexity of the distribution in question.    Surprisingly, despite an intense history of study by the statistics, information theory, and computer science communities, aspects of basic hypothesis testing and estimation questions--especially in the undersampled regime---remain unresolved, and require both new algorithms, and new analysis techniques.

In this work, we examine the basic hypothesis testing question of deciding whether two unknown  distributions over discrete supports are identical (or extremely similar), versus have total variation distance at least $\eps$, for some specified parameter $\eps>0$.    We consider (and largely resolve) this question in the extremely practically relevant setting of \emph{unequal sample sizes}.  Informally, taking $\eps$ to be a small constant, we show that provided $p$ and $q$ are supported on at most $n$ elements, for any $\gamma \in [0,1/3],$ the hypothesis test can be successfully performed (with high probability over the random samples) given samples of size $m_1=\Theta(n^{2/3 + \gamma})$  from $p$, and $m_2 = \Theta(n^{2/3 - \gamma/2})$ from $q$.  Furthermore, for every $\gamma$ in this range, this tradeoff between $m_1$ and $m_2$ is necessary, up to constant factors.   Thus our results smoothly interpolate between the known bounds of $\Theta(n^{2/3})$  on the sample size necessary in the setting where one is given two equal-sized samples~\cite{batu_conference,chan_valiant}, and the bound of $\Theta(\sqrt{n})$  on the sample size in the setting in which the sample is drawn from one distribution and the other distribution is \emph{known} to the algorithm~\cite{paninsky,instance_optimal}.  Throughout most of the regime of parameters, when $m_1 \ll m_2^2$, our algorithm is a natural extension of the algorithm proposed in~\cite{chan_valiant}, and is similar to the algorithm proposed in~\cite{orlitsky_unequal} except with the addition of a normalizing term.  In the extreme regime when $m_1 \approx n,$ our algorithm requires an additional statistic which appears to be new.  Throughout the regime of parameters, our algorithm is relatively simple, and appears to be practically viable.  In section~\ref{sec:discussion} we illustrate the efficacy of our approach on both synthetic data, and on the real-world problem of deducing whether two words are synonyms, based on a small sample of the bi-grams in which they occur.  

We also note that, as pointed out in several related works~\cite{orlitsky_unequal,goldreich_ron,batu_conference}, this hypothesis testing question has several applications to other problems, such as estimating or testing the mixing time of Markov processes, and our results yield improved algorithms in these settings. 

\subsection{Related Work}

 The general question of how to estimate or test properties of distributions using fewer samples than would be necessary to actually learn the distribution, has been studied extensively since the late '90s.  Most of the work has focussed on ``symmetric'' properties (properties whose value is invariant to relabeling domain elements) such as entropy, support size, and distance metrics between distributions (such as $\ell_1$ distance).  This has included both algorithmic work (e.g.~\cite{rk_lb,batu_entropy,batu_independence, charikar_distinct_values,guha_entropy, paninsky_mutual_information,paninsky_entropy,gpvaliant_clt,vv_nips,instance_optimal,gpvaliant_power}), and results on developing techniques and tools for establishing lower bounds (e.g. ~\cite{distinct_element,pvaliant_stoc,gpvaliant_clt}).   See the recent survey by Rubinfeld for a more thorough summary of the developments in this area~\cite{rubinfeld_big}).

The specific problem of ``closeness testing'' or ``identity testing'', that is, deciding whether two distributions, $p$ and $q$, are similar, versus have significant distance, has two main variants:  the \emph{one-unknown-distribution} setting in which $q$ is known and a sample is drawn from $p$,  and the \emph{two-unknown-distributions} settings in which both $p$ and $q$ are unknown and samples are drawn from both.  We briefly summarize the previous results for these two settings.

In the one-unknown-distribution setting (which can be thought of as the limiting setting in the case that we have an arbitrarily large sample drawn from distribution $q$, and a relatively modest sized sample from $p$), initial work of Goldreich and Ron~\cite{goldreich_ron} considered the problem of testing whether $p$ is the uniform distribution over $[n]$, versus has distance at least $\eps$.  The tight bounds of $\Theta(\sqrt{n} / \eps^2)$ were later shown by Paninski~\cite{paninsky}, essentially leveraging the birthday paradox and the intuition that, among distributions supported on $n$ elements, the uniform distribution minimizes the number of domain elements that will be observed more than once.  Batu et al.~\cite{batu_independence} showed that, up to polylogarithmic factors of $n$, and polynomial factors of $\eps$, this dependence was optimal for worst-case distributions over $[n]$.  Recently, an ``instance--optimal'' algorithm and matching lower bound was shown: for any distribution $q$, up to constant factors, $\max\{\frac{1}{\eps}, \varepsilon^{-2}||q_{-\Theta(\eps)}^{-\max}||_{2/3}\}$ samples from $p$ are both necessary and sufficient to test $p=q$ versus $||p-q|| \ge \eps$, where $||q_{-\Theta(\eps)}^{-\max}||_{2/3} \le ||q||_{2/3}$ is the 2/3-rd norm of the vector of probabilities of distribution $q$ after the maximum element has been removed, and the smallest elements up to $\Theta(\eps)$ total mass have been removed.  (This immediately implies the tight bounds that if $q$ is any distribution supported on $[n]$, $O(\sqrt{n}/\eps^2)$ samples are sufficient to test its identity.

The two-unknown-distribution setting was introduced to this community by Batu et al. \cite{batu_conference} (refer to \cite{batu} for the journal version), and using {\it collision statistics}, they proposed an algorithm that requires $m=O(\varepsilon^{-8/3}n^{2/3}\log n)$  samples from each distribution. Later, Valiant \cite{pvaliant_stoc} proved a lower bound of $m=\Omega(n^{2/3})$, which was tight up to logarithmic factors in $n$.  Recently, Chan et al. \cite{chan_valiant} determined the optimal sample complexity for this problem: they showed that $m=\Theta(\max\{n^{2/3}/\varepsilon^{4/3}, \sqrt n/\varepsilon^2\})$ samples are necessary and sufficient for closeness testing, up to constant factors.  In a slightly different vein, Acharya et al. \cite{orlitsky,orlitsky_classification} recently considered the question of closeness testing with two unknown distributions from the standpoint of competitive analysis. They proposed an algorithm that performs the desired task using $O(n^{3/2}\polylog n)$ samples, and a lower bound of $\Omega(n^{7/6})$, where $n$ represents the number of samples required to determine whether a set of samples were drawn from $p$ versus $q$, in the setting where $p$ and $q$ are explicitly known.

A natural generalization of this hypothesis testing problem, which interpolates between the two-unknown-distribution setting and the one-unknown-distribution setting, is to consider unequal sized samples from the two distributions. More formally, given $m_1$ samples from the distribution $p$, the {\it asymmetric closeness testing} problem is to determine how many samples, $m_2,$ are required from the distribution $q$ such that the hypothesis $p=q$ versus $||p-q||_1> \varepsilon$ can be distinguished with large constant probability (say 2/3).  Note that  the results of Chan et al. \cite{chan_valiant} imply that it is sufficient to consider $m_1 \geq \Theta(\max\{n^{2/3}/\varepsilon^{4/3}, \sqrt n/\varepsilon^2\})$. This problem was studied recently by Acharya et al. \cite{orlitsky_unequal}: they gave an algorithm that given $m_1$ samples from the distribution $p$ uses $m_2=O(\max\{\frac{n\log n}{\varepsilon^3\sqrt m_1}, \frac{\sqrt{n\log n}}{\varepsilon^2}\})$ samples from $q$, to distinguish the two distributions with high probability. They also proved a lower bound of $m_2=\Omega(\max\{\frac{\sqrt n}{\varepsilon^2}, \frac{n^2}{\varepsilon^4 m_1^2}\})$.  There is a polynomial gap in these upper and lower bounds in the dependence on $n$, $\sqrt m_1$ and $\varepsilon$. 

As a corollary to our main hypothesis testing result, we obtain an improved algorithm for testing the mixing time of a Markov chain.  The idea of testing mixing properties of a Markov chain goes back to the work of Goldreich and Ron \cite{goldreich_ron}, which conjectured an algorithm for testing expansion of bounded-degree graphs.  Their test is based on picking a random node and testing whether random walks from this node reach a distribution that is close to the uniform distribution on the nodes of the graph. They conjectured that their algorithm had $O(\sqrt n)$ query complexity. Later, Czumaj and Sohler \cite{czumaj}, Kale and Seshadhri \cite{kale_seshadhri}, and Nachmias and Shapira \cite{nachmias} have independently concluded that the algorithm of Goldreich and Ron is provably a test for expansion property of graphs. Rapid mixing of a chain can also be tested using eigenvalue computations. Mixing is related to the separation between the two largest eigenvalues \cite{jerrum_sinclair,levinpereswilmer}, and eigenvalues of a dense $n\times n$ matrix can be approximated in $O(n^3)$ time and $O(n^2)$ space. However, for a sparse $n\times n$ symmetric matrix with $m$ nonzero entries, the same task can be achieved in $O(n(m+\log n))$ operations and $O(n+m)$ space.   Later, Batu et al. \cite{batu} used their $\ell_1$ distance test on the $t$-step distributions, to test mixing properties of Markov chains. Given a finite Markov chain with state space $[n]$ and transition matrix $\pmb P=((P(x, y)))$, they essentially show that one can estimate the mixing time $\tau_{mix}$ up to a factor of $\log n$ using $\tilde{O}(n^{5/3} \tau_{mix})$ queries to a \emph{next node} oracle, which takes a state $x\in [n]$ and outputs the state $y\in [n]$  drawn from the probability $P(x, y)$.  Such an oracle can often be simulated significantly more easily than actually computing the transition matrix $P(x,y)$.

We conclude this related work section with a comment on ``robust'' hypothesis testing and distance estimation.  A natural hope would be to simply estimate $||p-q||$ to within some additive $\eps$, which is a strictly more difficult task than distinguishing $p=q$ from $||p-q|| \ge \eps$.  The results of Valiant and Valiant~\cite{gpvaliant_clt,vv_nips,instance_optimal,gpvaliant_power} show that this problem is  significantly more difficult than hypothesis testing: the distance can be estimated to additive error $\eps$ for distributions supported on $\le n$ elements using samples of size $O(n/\log n)$ (in both the setting where either one, or both distributions are unknown). Moreover, $\Omega(n/\log n)$ samples are  information theoretically necessary, even if $q$ is the uniform distribution over $[n]$, and one wants to distinguish the case that $||p-q||_1 \le \frac{1}{10}$ from the case that $||p-q||_1 \ge \frac{9}{10}.$  Recall that the non-robust test of distinguishing $p=q$ versus $||p-q|| >9/10$ requires a sample of size only $O(\sqrt{n})$.  The exact worst-case sample complexity of distinguishing whether $||p-q||_1 \le \frac{1}{n^{c}}$ versus $||p-q||_1 \ge \eps$ is not well understood, though in the case of constant $\eps$, up to logarithmic factors, the required sample size seems to scale linearly in the exponent between $n^{2/3}$ and $n$ as $c$ goes from $1/3$ to $0$.

\subsection{Our results}

Our main result resolves the closeness testing problem in the unequal sample setting, to constant factors, in terms of the worst-case distributions of support size $\le n$:

\begin{theorem}\label{thm:i}
Given $m_1\geq n^{2/3}/\varepsilon^{4/3}$ and $\varepsilon>n^{-1/12}$, and sample access to distributions $p$ and $q$ over $[n]$, there is an $O(m_1)$ time algorithm which takes $\Theta(m_1)$ samples from $p$ and $m_2=O(\max\{\frac{n}{\sqrt m_1\varepsilon^2}, \frac{\sqrt n}{\varepsilon^2}\})$ samples from $q$, and with probability at least 2/3 distinguishes whether 
\begin{equation}
||p-q||_1\leq O\left(\frac{1}{m_2}  \right) \quad \text{versus} \quad ||p-q||_1 \geq \varepsilon.
\label{eq:testing_problem}
\end{equation} 
Moreover, given $\Theta(m_1)$ samples from $p$, $\Omega(\max\{\frac{n}{\sqrt m_1\varepsilon^2}, \frac{\sqrt n}{\varepsilon^2}\})$ samples from $q$ are information-theoretically necessary to distinguish $p=q$ from $||p-q||_1 \ge \eps$ with any constant probability bounded above by $1/2$.
\label{th:sample_size}
\end{theorem}

The lower bound in the above theorem is proved using the machinery developed in Valiant \cite{pvaliant_stoc}, and ``interpolates'' between the $\Theta(\sqrt{n}/\eps^2)$ lower bound in the one-unknown-distribution setting of testing uniformity~\cite{paninsky} and the $\Theta(n^{2/3}/\eps^{4/3})$ lowerbound in the setting of equal sample sizes from two unknown distributions~\cite{chan_valiant}. The upper bound is proved in several steps. We begin by proposing two algorithms for the hypothesis testing problem $p=q$ versus $||p-q||_1>\varepsilon$ depending on the value of $m_1$: the {\it non-extreme} regime, that is, $m_1=O((n/\varepsilon^2)^{1-\gamma})$, and the {\it extreme} case where $m_1=O(n)$.  In the non-extreme regime, our algorithm is an extension of the algorithm proposed in~\cite{chan_valiant}, and is similar to the algorithm proposed in~\cite{orlitsky_unequal} except with the addition of a normalizing term.  In the extreme regime when $m_1 \approx n,$ we incorporate an additional statistic that has not appeared before in the literature.\footnote{We note that a further extension of this algorithm yields a stronger robustness parameter, distinguishing between $||p-q||_1\leq O\left(\max\left(\frac{1}{\sqrt{m_1}},\frac{\eps}{\sqrt n} \right)  \right)$ versus $||p-q||_1 \geq \varepsilon$.}
\iffalse
In the appendix, we describe a slight modification of our testing algorithm that improves the robustness, and can distinguish the hypotheses testing problem (\ref{eq:testing_problem}). This yields a better sample complexity and slightly higher robustness parameter than was previously known. 
\fi

As an application of Theorem~\ref{thm:i} in the extreme regime when $m_1=O(n)$, we obtain an improved algorithm for estimating the mixing time of a Markov chain:

\begin{corollary} Consider a finite Markov chain with state space $[n]$ and a \emph{next node oracle}; there is an algorithm that estimates the mixing time, $\tau_{mix}$, up to a multiplicative factor of $\log n$, that uses $\tilde{O}(n^{3/2} \tau_{mix})$ time and queries to the next node oracle.
\label{cor:tmix}
\end{corollary}

It remains an intriguing open question whether this query complexity is optimal; we are not aware of any lower bounds beyond the trivial $\Omega(n \tau_{mix}).$

\subsection{Outline}
We begin by stating our testing algorithms, and describe both the intuition behind the algorithms, as well as the high level proof approach.  Throughout the theoretical portion of the paper, we will work in the ``Poissonized'' setting, where we assume that we have access to $\dPois(m_1)$ samples from distribution $p$, and $\dPois(m_2)$ samples drawn distribution $q$.  This assumption that the sample size is a random variable renders the number of occurrences of different domain elements independent.  Because $\dPois(\lambda)$ is tightly concentrated about its expectation, both the upper and lower bounds on the sample complexities proved in this ``Poissonized'' setting also hold (up to factors of $1\pm o(1)$) in the setting in which one obtains samples of a fixed size.

The complete proofs require rather involved calculations of the moments of the various statistics employed by our algorithms, and are deferred to Appendix~\ref{app:moments}.  The applications of our testing results to the problem of testing or estimating the mixing time of a Markov chain is discussed in Section~\ref{sec:tmix}.   Finally, Section~\ref{sec:discussion} contains some empirical results, suggesting that the statistic at the core of our algorithms performs very well in practice.  This section contains both results on synthetic data, as well as an illustration of how to apply these ideas to the problem of estimating some notion of the semantic similarity of two words based on samples of the $n$-grams that contain the words in a corpus of text.    The construction and proof of our lower bounds, showing the optimality of our testing algorithms is given in Appendix~\ref{sec:lb}.

\section{Algorithms for $\ell_1$ Testing}

In this section we describe algorithms for $\ell_1$ testing with unequal samples, which give the upper bound in Theorem \ref{th:sample_size}. We propose two algorithms depending on the value of $m_1$: the {\it non-extreme} regime, that is, $m_1=O((n/\varepsilon^2)^{1-\gamma})$, and the {\it extreme} case where $m_1 \approx n$.  

\subsection{Algorithms for $\ell_1$ Testing: Non-Extreme Case}

We begin with the basic algorithm (Algorithm \ref{frame0}), which is optimal in the non-extreme regime, for constant $\varepsilon$.  All the subsequent algorithms are modifications of this basic algorithm.

\iffalse
\textcolor{blue}{We wish to test $p=q$ versus $||p-q||_1>\varepsilon$ with probability at least 2/3. The samples  drawn $p$  and $q$ are partitioned into two sets as follows: Take two independent samples $\mathcal S_1$ and $\mathcal S_2$ of size $\dPois(m_1)$ from $p$, and two independent samples $\mathcal T_1$ and $\mathcal T_2$ of size $\dPois(m_2)$ from $q$. For $i\in [n]$, denote by $\hat X_i, \hat Y_i$ the number of times the domain element $i\in [n]$ occurs in the sample $\mathcal S_1$  and $\mathcal T_1$, respectively. Let $\hat p_i=\hat X_i/m_1$ and $\hat q_i=\hat Y_i/m_2$. Similarly, define $X_i, Y_i$ the number of times the domain element $i\in [n]$ occurs in the sample $\mathcal S_2$ and $\mathcal T_2$, respectively.% and $\hat p_i=X_i/m_1$ and $\hat q_i=Y_i/m_2$.
}
\fi

%The algorithm for the closeness testing problem has two variations depending on the magnitude of $m_1$.  In this section we describe the algorithm for the {\it non-extreme case}, that is, $m_1=O(\left(n/\varepsilon^2\right)^{1-\gamma})$, for some $\gamma >0$. 

%The basic algorithm for asymmetric closeness testing can be described easily for the case where $\varepsilon =\Omega (1)$. In this case the algorithm is very simple and can be easily implemented in applied problems. 

\begin{algorithm}[H]

\begin{flushleft}
Suppose $\varepsilon=\Omega(1)$ and $m_1=O(n^{1-\gamma})$ for some $\gamma>0$. Let $S_1,S_2$ denote two independent sets of $\dPois(m_1)$ samples from $p$ and let $T_1,T_2$ denote two independent sets of $\dPois(m_2)$ samples drawn from $q$.   We wish to test $p=q \text{ versus } ||p-q||_1> \varepsilon.$

\begin{itemize}
\item Let $b= \frac{256 \log n}{\eps^2 m_2}$, and define the set $B = \{i\in [n]: \frac{X^{S_1}_i}{m_1} >b\} \cup  \{i\in [n]: \frac{Y^{T_1}_i}{m_2} >b\}$, where $X^{S_1}_i$ denotes the number of occurrences of $i$ in $S_1$, and $Y^{T_1}_i$ denotes the number of occurrences of $i$ in $T_1$.
\item Let $X_i$ denote the number of occurrences of element $i$ in $S_2$, and $Y_i$ denote the number of occurrences of element $i$ in $T_2$:
\end{itemize}

\end{flushleft}

\begin{enumerate}

\item Check if 
\begin{equation}
\sum_{i\in B}\left|\frac{X_i}{m_1}-\frac{Y_i}{m_2}\right|\leq \varepsilon/6.
\label{eq:test01}
\end{equation}

\item Check if
\begin{equation}
\sum_{i \in [n] \setminus B}\frac{(m_2X_i-m_1Y_i)^2-(m_2^2X_i+m_1^2Y_i)}{X_i+Y_i}\leq C_\gamma m_1^{3/2}m_2,
\label{eq:test02}
\end{equation}
for an appropriately chosen constant $C_\gamma$ (depending on $\gamma$).

\item If (\ref{eq:test01}), and (\ref{eq:test02}) hold, then ACCEPT. Otherwise, REJECT. 

\end{enumerate}
\caption{Closeness Testing: Non-Extreme Case (The Basic Algorithm)}
\label{frame0}
\end{algorithm}

The intuition behind the above algorithm is as follows: with high probability, all elements in the set $B$ satisfy either $p_i > b/2,$ or $q_i > b/2$ (or both).  Given that these elements are ``heavy'', their contribution to the $\ell_1$ distance will be accurately captured by the $\ell_1$ distance of their empirical frequencies (where these empirical frequencies are based on the second set of samples, $S_2,T_2.$    For the elements that are not in set $B$---the ``light'' elements---we use a modification of the statistic used by Chan at al. \cite{chan_valiant}, where the terms are re-weighted according to the unequal sample sizes. This is similar to the algorithm proposed in~\cite{orlitsky_unequal}, where instead of (\ref{eq:test02}) the authors used the numerator of (\ref{eq:test02})  to distinguish the light elements. However, just using the numerator only gives an estimate of the $\ell_2$ distance between $p$ and $q$. The normalization by $X_i+Y_i$ in (\ref{eq:test02}) ``linearizes" the statistic, which gives some estimate of the $\ell_1$ distance between the two distributions for the light elements. Similar results can possibly be obtained by using other linear functions of $X_i$ and $Y_i$ in the denominator, though we note that the ``obvious'' normalizing factor of $X_i+\frac{m_1}{m_2}Y_i$ does not seem to work theoretically, and seems to have extremely poor performance in practice.   Additionally, the unweighted $X_i+Y_i$ normalization is easier to analyze. 

Finally, we should emphasize that the crude step of using two independent batches of samples---the first to obtain the partition of the domain into ``heavy'' and ``light'' elements, and the second to actually compute the statistics, is for ease of analysis.  As our empirical results of Section~\ref{sec:discussion} suggest, for practical applications one might want to use only the $Z$-statistic of  (\ref{eq:test02}), and one certainly should not ``waste'' half the samples to perform the ``heavy''/``light'' partition.
\iffalse

, and one might wish to use only the statistic of  to only compute the statistic of 
\textcolor{blue}{Moreover, partitioning the sample into two independent copies ensures that the set $B_0$ is independent of the samples $\mathcal S_2$ and $\mathcal T_1$, which makes the analysis tractable.}\fi

To get the optimal dependence on $\varepsilon$, the above algorithm needs to be slightly modified. Algorithm \ref{frame1} gives the optimal sample complexity in the non-extreme case, for any $\varepsilon \geq n^{-\frac{1}{12}}$.   We state the algorithm here, as the algorithm in the extreme case where $m_1 \approx n$ and $m_2 \approx \sqrt{n}$ leverages some of its components.  The analysis of the algorithm and the proof of the following proposition are given in Appendix \ref{app:ppn:testing_ub_I}.

\begin{algorithm}[H] 
\begin{flushleft}
Suppose $m_1=O(\left(n/\varepsilon^2\right)^{1-\gamma}) \le n$ for some $\gamma>0$.  Let $S_1,S_2$ denote two independent sets of $\dPois(m_1)$ samples from $p$ and let $T_1,T_2$ denote two independent sets of $\dPois(m_2)$ samples drawn from $q$.   We wish to test $p=q \text{ versus } ||p-q||_1> \varepsilon.$

\begin{itemize}
\item Let $b= \frac{256 \log n}{\eps^2 m_2}$, and $b'= \frac{256 \log n}{m_2}$, and let $X^{S_1}_i$ denote the number of occurrences of $i$ in $S_1$, and $Y^{T_1}_i$ denote the number of occurrences of $i$ in $T_1$.
\item Define the ``heavy'' set $B = \{i\in [n]: \frac{X^{S_1}_i}{m_1} >b\} \cup  \{i\in [n]: \frac{Y^{T_1}_i}{m_2} >b\}$.
\item Define the ``medium'' set $M=\left\{i \in [n]: b'\leq \max\{\frac{X^{S_1}_i}{m_1}, \frac{Y^{T_1}_i}{m_2}\}\leq b\right\}$.
\item Define the ``light'' set $H = [m] \setminus (B \cup M).$
\item Let $X_i$ denote the number of occurrences of element $i$ in $S_2$, and $Y_i$ denote the number of occurrences of element $i$ in $T_2$:
\end{itemize}
\end{flushleft}

\begin{enumerate}

\item Check if 
\begin{equation}V_{B}:=\sum_{i\in B}V_i:=\sum_{i\in B}\left|\frac{X_i}{m_1}-\frac{Y_i}{m_2}\right|\leq \varepsilon/6.
\label{eq:test1}
\end{equation}

\item Check if
\begin{equation}W_{M}:=\sum_{i\in M}W_i:=\sum_{i\in M}(m_2X_i-m_1Y_i)^2-(m_2^2X_i+m_1^2Y_i]\leq \frac{\varepsilon^2 m_1^2 m_2\log n}{2}.
\label{eq:test2}
\end{equation}

\item Check if
\begin{equation}
Z_{ H}:=\sum_{i\in H}Z_i:=\sum_{i \in H}\frac{(m_2X_i-m_1Y_i)^2-(m_2^2X_i+m_1^2Y_i)}{X_i+Y_i}\leq C_\gamma m_1^{3/2}m_2,.
\label{eq:test3}
\end{equation}
Where $C_\gamma$ is an appropriately chosen absolute constant, dependent on $\gamma.$

\item If (\ref{eq:test1}), (\ref{eq:test2}), and (\ref{eq:test3}) hold, then ACCEPT. Otherwise, REJECT. 

\end{enumerate}
\caption{Asymmetric Closeness Testing: Non-Extreme Case}
\label{frame1}
\end{algorithm}

\begin{proposition}
Suppose $m_1=O(\left(n/\varepsilon^2\right)^{1-\gamma}) \le n$ for some $\gamma>0$, and $\varepsilon>n^{-1/12}$. Then algorithm (\ref{frame1}) takes $\Theta(m_1)$ samples from $p$ and $O(\max\{\frac{n}{\sqrt m_1\varepsilon^2}, \frac{\sqrt n}{\varepsilon^2}\})$ samples from $q$, and with probability at least 2/3 distinguishes whether $p = q$ versus $||p-q||_1 \geq \varepsilon$. 
\label{ppn:testing_ub_I}
\end{proposition}

\iffalse
The  {\it robust asymmetric closeness} testing problem with robustness parameter of $\varepsilon_n'$ is to distinguish the hypothesis $||p-q||_1\leq \varepsilon'_n$ versus $||p-q||_1> \varepsilon$ with probability at least 2/3. It turns out, under the assumptions of Theorem \ref{th:sample_size} and with the same sample complexity, Algorithm \ref{frame1} has a robustness parameter of $\varepsilon_n':=O\left(\max\left\{\frac{1}{\sqrt m_1}, \frac{\varepsilon}{\sqrt n}\right\}\right)$ in the non-extreme-case.  The details of the analysis are given in Corollary \ref{cor:sample_size_robust} in Appendix \ref{app:robust_ne}.

\fi

\subsection{Algorithm for $\ell_1$ Testing: Extreme Case}

For the extreme case, $m_1 \approx n$ and $m_2 \approx \sqrt{n}$, the re-weighted statistic $Z_{H}$ might have large variance, necessitating a modification to the algorithm in this extreme case. To see the cause of such variance, consider the case where the samples are drawn from the uniform distribution, $\dUnif[n]$.  By the birthday paradox, we might see a constant number of indices $i$ for which $Y_i=2,$ but $X_i=0$.  Such domain elements themselves contribute $O(n^4)$ to the variance of $Z_H$, which is at the threshold of what can be tolerated.  The statistic (\ref{eq:R_N}) introduced below, is tailored to deal with these cases, and captures the intuition that we are more tolerant of indices $i$ for which $Y_i=2$ if the corresponding $X_i$ is larger.
\iffalse 
Moreover, for $p=q$, there cannot be any $i\in [n]$ with $Y_i\geq 3$ for which the corresponding $X_i$ is ``too small". Step 1 of Algorithm \ref{frame2} is designed to distinguish the two hypotheses based on the occurrence of  such ``unbalanced'' samples. 
\fi

These modifications allow us to solve the closeness testing problem in the extreme case. In fact, the following algorithm works whenever $\Omega(\left(n/\varepsilon^2\right)^{8/9+\gamma})$, overlapping with the non-extreme case for $\gamma \in (0, 1/9)$.

\begin{algorithm}[H]
\begin{flushleft}
Suppose $m_1=\Omega(\left(n/\varepsilon^2\right)^{8/9+\gamma})$ for some $\gamma>0$.  Let $S_1,S_2$ denote two independent sets of $\dPois(m_1)$ samples from $p$ and let $T_1,T_2$ denote two independent sets of $\dPois(m_2)$ samples drawn from $q$.   We wish to test $p=q \text{ versus } ||p-q||_1> \varepsilon.$

\begin{itemize}
\item Define $b,b',B,M,H$ as in Algorithm~\ref{frame1}.
\iffalse
 Let $b= \frac{256 \log n}{\eps^2 m_2}$, and $b'= \frac{256 \log n}{m_2}$, and let $X^{S_1}_i$ denote the number of occurrences of $i$ in $S_1$, and $Y^{T_1}_i$ denote the number of occurrences of $i$ in $T_1$.
\item Define the ``heavy'' set $B = \{i\in [n]: \frac{X^{S_1}_i}{m_1} >b\} \cup  \{i\in [n]: \frac{Y^{T_1}_i}{m_2} >b\}$.
\item Define the ``medium'' set $M=\left\{i \in [n]: b'\leq \max\{\frac{X^{S_1}_i}{m_1}, \frac{Y^{T_1}_i}{m_2}\}\leq b\right\}$.
\item Define the ``light'' set $H = [m] \setminus (B \cup M).$  \fi
\item Let $X_i$ denote the number of occurrences of element $i$ in $S_2$, and $Y_i$ denote the number of occurrences of element $i$ in $T_2$: 
\end{itemize}
\end{flushleft}

\begin{enumerate}

\item REJECT if there exists $i\in[n]$ such that $Y_i\geq 3$ and $X_i\leq \frac{m_1\varepsilon^{2/3}}{10m_2n^{1/3}}$.\iffalse =\frac{m_1^{3/2}\varepsilon^{8/3}}{n^{4/3}}=\Omega(n^\gamma)$,\fi

\item Check if
\begin{equation}
R_{H}:=\sum_{i\in H} \frac{\pmb 1\{Y_i=2\}}{X_i+1}\leq C_1\frac{m_2^2}{m_1}, 
\label{eq:R_N}
\end{equation}
where $C_1$ is an appropriately chosen absolute constant.

\item If step (1)  is not rejected and (\ref{eq:test1}), (\ref{eq:test2}), (\ref{eq:test3}), and (\ref{eq:R_N}) are satisfied, then ACCEPT. Otherwise, REJECT.
\end{enumerate}
\caption{Asymmetric Closeness Testing: Extreme Case}
\label{frame2}
\end{algorithm}

Proposition \ref{ppn:testing_lb} below summarizes the performance of the above algorithm. The proof is given in Appendix \ref{app:ppn:testing_ub_II}.

\begin{proposition}
Suppose  $m_1=\Omega(\left(n/\varepsilon^2\right)^{8/9+\gamma})$ for some $\gamma>0$ and $\varepsilon>n^{-1/12}$. Then algorithm (\ref{frame2}) takes $\Theta(m_1)$ samples from $p$ and $O(\max\{\frac{n}{\sqrt m_1\varepsilon^2}, \frac{\sqrt n}{\varepsilon^2}\})$ samples from $q$, and with probability at least 2/3 distinguishes whether $p = q$ versus $||p-q||_1 \geq \varepsilon$. 
\label{ppn:testing_ub_II}
\end{proposition}

It is worth noting that one can also define a natural analog of the $R_H$ statistic corresponding to the indices $i$ for which $Y_i = 3$, etc., and that the use of such statics improves the robustness parameter of the test.

\iffalse
such statistics can yield
To get an algorithm with a robustness factor of $\varepsilon/\sqrt n$ in the extreme case, Algorithm \ref{frame2} has to be further modified. This involves defining a statistic similar to (\ref{eq:R_N}) for the indices $Y_i=3$. The modified algorithm, which leads to the following  proposition, is presented in Appendix \ref{app:ppn:re}.

\begin{proposition}
Given $\varepsilon>n^{-1/12}$, there exists an algorithm which takes $\Theta(n)$ samples from $p$ and $O(\frac{\sqrt n}{\varepsilon^2})$ samples from $q$, and with probability at least 2/3 distinguishes whether $||p-q||_1\leq \frac{\varepsilon}{\sqrt n}$ versus  $||p-q||_1 \geq \varepsilon $.
\label{ppn:re}
\end{proposition}

Propositions \ref{ppn:testing_ub_I}, \ref{ppn:testing_ub_II} and \ref{ppn:re}, and Corollary 
\ref{cor:sample_size_robust} together establish the upper bound in Theorem \ref{th:sample_size} for the closeness testing problem with unequal sample sizes.
\fi

\section{Estimating Mixing in Markov Chains}
\label{sec:tmix}

Consider a finite Markov chain with state space $[n]$, transition matrix $\pmb P=((P(x, y)))$, with stationary distribution $\pi$.  The {\it $t$-step distribution starting at the point $x\in [n]$}, $P_x^t(\cdot)$ is the probability distribution on $[n]$ obtained by running the chain for $t$ steps starting from $x$. More formally, for $A\subseteq [n]$, $P_x^t(A)=\Pr[X_t\in A|X_0=x]$,
where $(X_0, X_1, \ldots, X_t)$ are the steps of the chain. The $t$-step distribution $P_x^t$ can be computed as a vector matrix product $\vec e_x \pmb P^t$, where $\vec e_x\in \R^n$ is the standard basis vector which has 1 at position $x$ and zeros everywhere else. 

\begin{definition}
The $\varepsilon$-{\it mixing time} of a Markov chain with transition matrix $\pmb P=((P(x, y)))$ is defined as $t_{\mathrm{mix}} (\varepsilon):=\inf\left\{t \in [n]: \sup_{x\in [n]}\frac{1}{2}\sum_{y\in [n]}|P^t_x(y)-\pi(y)|\leq \varepsilon\right\}$. \end{definition}
\vspace{-.5cm}
\begin{definition}The {\it average $t$-step distribution} of a Markov chain $\pmb P$ with $n$ states is the distribution $\overline P^t=\frac{1}{n}\sum_{x\in [n]}P^t_x(A)$, that is, the distribution obtained by choosing $x$ uniformly from $[n]$ and walking $t$ steps from the state $x$. 
\end{definition}

As observed by Batu et al.~\cite{batu}, $\ell_1$ closeness testing can be used to test whether a Markov chain is close to mixing after some specified number of steps, $t_0$.  Here, we note that \emph{asymmetric} closeness testing (as opposed to the case of equal sized samples as employed in~\cite{batu}), yields an improvement in the performance of the testing algorithm for Markov chain mixing.

The algorithm to test mixing proposed by Batu et al. \cite{batu} involves testing the $\ell_1$ difference between distributions $P^{t_0}_x$ and $\overline P^{t_0}$, for every $x\in [n]$. The algorithm uses their $\ell_1$ distance test which draws $\tilde{O}(n^{2/3}\log n)$ samples from both the distributions $P^{t_0}_x$ and $\overline P^{t_0}$, and has a overall running time of $\tilde{O}(n^{5/3} t_0)$. However, the distribution $\overline P^{t_0}$ does not depend to the starting state $x$ and using Algorithm \ref{frame2}, it suffices to take $\tilde{O}(n)$ samples from $\overline P^{t_0}$ once and $\tilde{O}(\sqrt n)$ samples from $P^t_x$, for every $x\in [n]$. This results in a query and runtime complexity of $\tilde{O}(n^{3/2} t_0).$

\begin{algorithm}[H]
\begin{flushleft}

Given $t_0\in \R$ and a finite Markov chain with state space $[n]$ and transition matrix $\pmb P=((P(x, y)))$, we wish to test
 
\begin{equation}
H_0: t_{\mathrm{mix}}\left(O\left(\frac{\varepsilon^2}{\sqrt n}\right)\right)\leq t_0, \quad \text{versus} \quad H_1:t_{\mathrm{mix}}\left(\varepsilon \right)> t_0.
\label{eq:tmix}
\end{equation}

\end{flushleft}

\begin{enumerate}
\item Draw $O(\log n)$ samples $S_1,\ldots,S_{O(\log n)},$ each of size $\dPois(C_1n)$ from the average $t_0$-step distribution.

\item For each state $x\in [n]$ we will distinguish whether $||P^{t_0}_x-\overline P^{t_0}||_1 \leq O(\frac{\varepsilon^2}{\sqrt n}),$ versus $||P^{t_0}_x-\overline P^{t_0}||_1> \varepsilon,$ with probability of error $\ll 1/n$.  We do this by running $O(\log n)$ runs of Algorithm~\ref{frame2}, with the $i$-th run using $S_i$ and a fresh set of $\dPois( O(\varepsilon^{-2}\sqrt n))$ samples from $P^t_x$. 

\item If all $n$ of the $\ell_1$ closeness testing problems are accepted, then we ACCEPT $H_0$.
\end{enumerate}
\caption{Testing for Mixing Times in Markov Chains}
\label{frame:tmix}
\end{algorithm}

\iffalse
The following proposition, proved in Appendix \ref{app:mixing}, characterizes the performance of the above algorithm.

\begin{proposition} Consider a finite Markov chain with state space $[n]$ and a \emph{next node oracle}  that takes a state $i \in [n]$ and outputs a state sampled according to the distribution of subsequent states. For $t_0>0$, there exists an algorithm with time and query complexity $\tilde{O}(n^{3/2}t_0)$ which distinguishes $t_{\mathrm{mix}}\left(1/64\sqrt n \right]\leq t_0$, versus $t_{\mathrm{mix}}(1/4) >  t_0$, with probability at least $2/3$, where $t_{\mathrm{mix}}(\delta)$ is the time for the chain to have distance at most $\delta$ from the stationary distribution.
\label{thm:tmix}
\end{proposition}
\fi

The above testing algorithm can be leveraged to \emph{estimate} the mixing time of a Markov chain, via the basic observation that if $t_{\mathrm{mix}}(1/4) \le t_0,$ then for any $\eps$, $t_{\mathrm{mix}}(\eps) \le \frac{\log \eps}{\log 1/2} t_0,$ and thus $t_{\mathrm{mix}}(1/\sqrt n) \le 2 \log n\cdot t_{\mathrm{mix}}(1/4).$  Because $t_{\mathrm{mix}}(1/4)$ and $t_{\mathrm{mix}}(O(1/\sqrt{n}))$ differ by at most a factor of $\log n$, by applying Algorithm~\ref{frame:tmix} for a geometrically increasing sequence of $t_0$'s, and repeating each test $O(\log t_0  + \log n)$ times, one obtains Corollary~\ref{cor:tmix}.

\section{Empirical Results}\label{sec:discussion}

Both our formal algorithms and the corresponding theorems involve some unwieldy constant factors (that can likely be reduced significantly).  Nevertheless, in this section we provide some evidence that the statistic at the core of our algorithms can be fruitfully used in practice, even for surprisingly small sample sizes.

\subsection{Testing similarity of words}
An extremely important primitive in natural language processing is the ability to estimate the \emph{semantic similarity} of two words.  Here, we show that the $Z$ statistic, $Z=\sum_i \frac{(m_2X_i- m_1Y_i)^2-(m_2^2X_i+m_1^2Y_i)}{m_1^{3/2}m_2(X_i+Y_i)}$, which is the core of our testing algorithms, can accurately distinguish whether two words are very similar based on surprisingly small samples of the contexts in which they occur.   Specifically, for each pair of words, $a,b$ that we consider, we select $m_1$ random occurrences of $a$ and $m_2$ random occurrences of word $b$ from the Google books corpus, using the Google Books Ngram Dataset.\footnote{The Google Books Ngram Dataset is freely available here: \url{http://storage.googleapis.com/books/ngrams/books/datasetsv2.html}}  We then compare the sample of words that follow $a$ with the sample of words that follow $b$.  Henceforth, we refer to these as samples of the set of bi-grams involving each word, although for convenience, we only considered the bigrams whose first word was the word in question.

Figure~\ref{fig1} illustrates the $Z$ statistic for various pairs of words that range from rather similar words like ``smart'' and ``intelligent'', to essentially identical word pairs such as ``grey'' and ``gray'' (whose usage differs mainly as a result of historical variation in the preference for one spelling over the other); the sample size of bi-grams containing the first word is fixed at $m_1=1,000,$ and the sample size corresponding to the second word varies from $m_2=50$ through $m_2=1,000$.  To provide a frame of reference, we also compute the value of the statistic for independent samples corresponding to the same word (i.e. two different samples of words that follow ``wolf''); these are depicted in red.   For comparison, we also plot the total variation distance between the empirical distributions of the pair of samples, which does not clearly differentiate between pairs of identical words, versus different words, particularly for the smaller sample sizes.  
   
\begin{figure}\centering
\includegraphics[width=400pt]{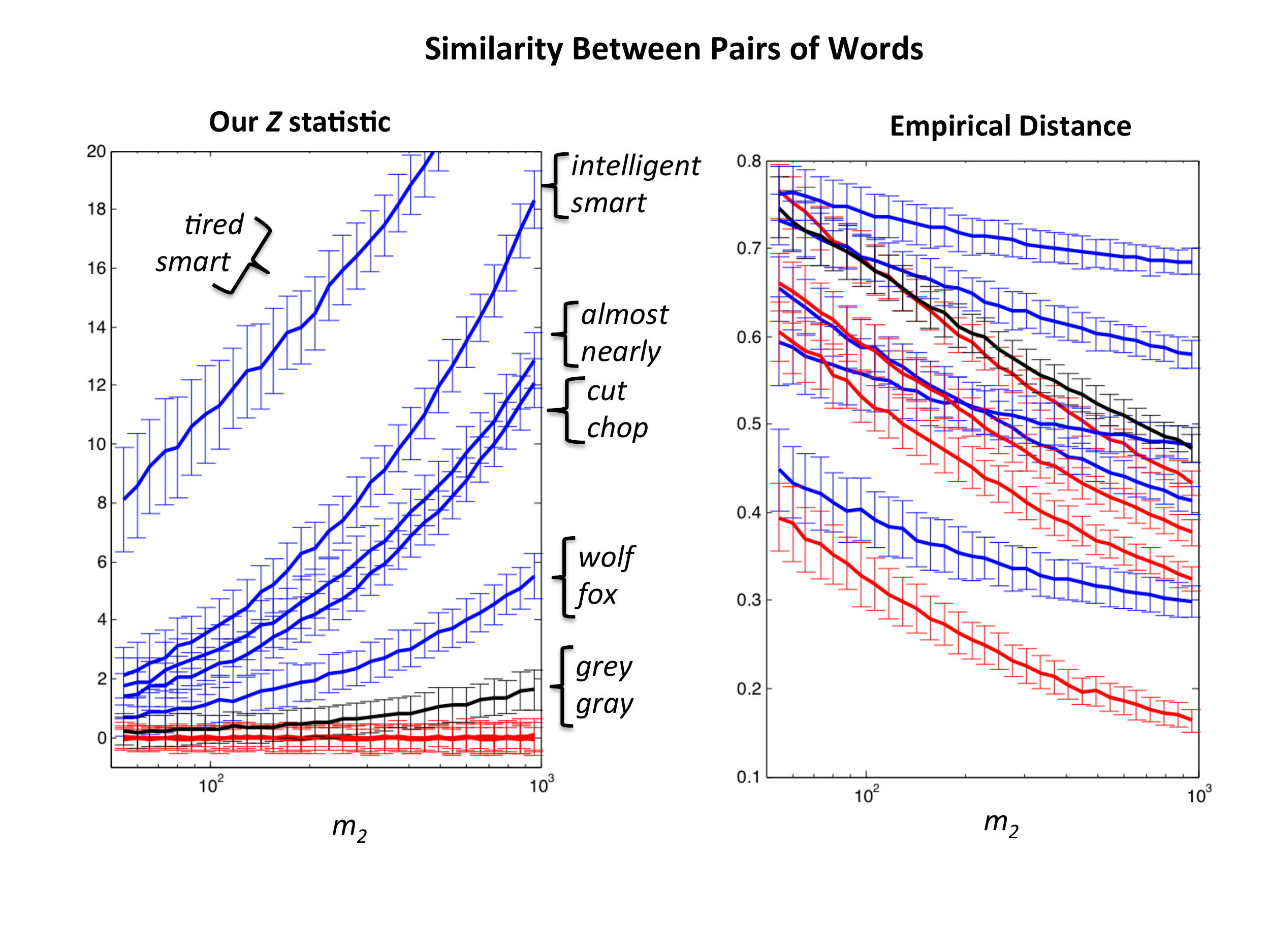}
\vspace{-1.6cm}
\caption{Two measures of the similarity between words, based on samples of the bi-grams containing each word.  Each line represents a pair of words, and is obtained by taking a sample of $m_1=1,000$ bi-grams containing the first word, and $m_2 =50,\ldots,1,000$ bi-grams containing the second word, where $m_2$ is depicted along the $x$-axis in logarithmic scale.  In both plots, the red lines represent pairs of identical words (e.g. ``wolf/wolf",``almost/almost'',\ldots).  The blue lines represent pairs of similar words (e.g. ``wolf/fox'', ``almost/nearly'',\ldots), and the black line represents the pair "grey/gray" whose distribution of bi-grams differ because of historical variations in preference for each spelling.   Solid lines indicate the average over 200 trials for each word pair and choice of $m_2$, with error bars of one standard deviation depicted.   The left plot depicts our statistic, which clearly distinguishes identical words, and demonstrates some intuitive sense of semantic distance.  The right plot depicts the total variation distance between the empirical distributions---which does not successfully distinguish the identical words, given the range of sample sizes considered.  The plot would not be significantly different if other distance metrics between the empirical distributions, such as f-divergence, were used in place of total variation distance.  Finally, note the extremely uniform magnitudes of the error bars in the left plot, as $m_2$ increases, which is a result of the $X_i+Y_i$ normalization term in the $Z$ statistic.  }\label{fig1}
\end{figure}

One subtle point is that the issue with using the empirical distance between the distributions goes beyond simply not having a consistent reference point.  For example, let $X$ denote a large sample of size $m_1$ from distribution $p$, $X'$ denote a small sample of size $m_2$ from $p$, and $Y$ denote a small sample of size $m_2$ from a different distribution $q$.  It might be tempting to hope that the empirical distance between $X$ and $X'$ will be smaller than the empirical distance between $X$ and $Y$.  As Figure~\ref{fig2} illustrates, this is not always the case, even for natural distributions: for this specific example, over much of the range of $m_2$, the empirical distance between $X$ and $X'$ is indistinguishable from that of $X$ and $Y$, and yet, as our statistic easy discerns, these distributions are very different.  

This point is further emphasized in Figure~\ref{fig3}, which depicts this phenomena in the synthetic setting where $p=\dUnif[5,000]$ is the uniform distribution over $5,000$ elements, and $q$ is the distribution whose elements have probabilities $(1\pm \eps)/5000$, for $\eps=1/4.$  The right plot represents the empirical probability that the distance between two empirical distributions of the samples from $p$ is larger than the  distance between the empirical distributions of the samples from $p$ and $q$; the left plot represents the analogous probability involving the $Z$ statistic.  In both plots, $m_1$ ranges between $n^{2/3}$ and $n$, and $m_2$ ranges between $n^{1/2}$ and $n$, for $n=5,000.$

\begin{figure}\centering
\vspace{-2cm}
\includegraphics[width=400pt,height=250pt]{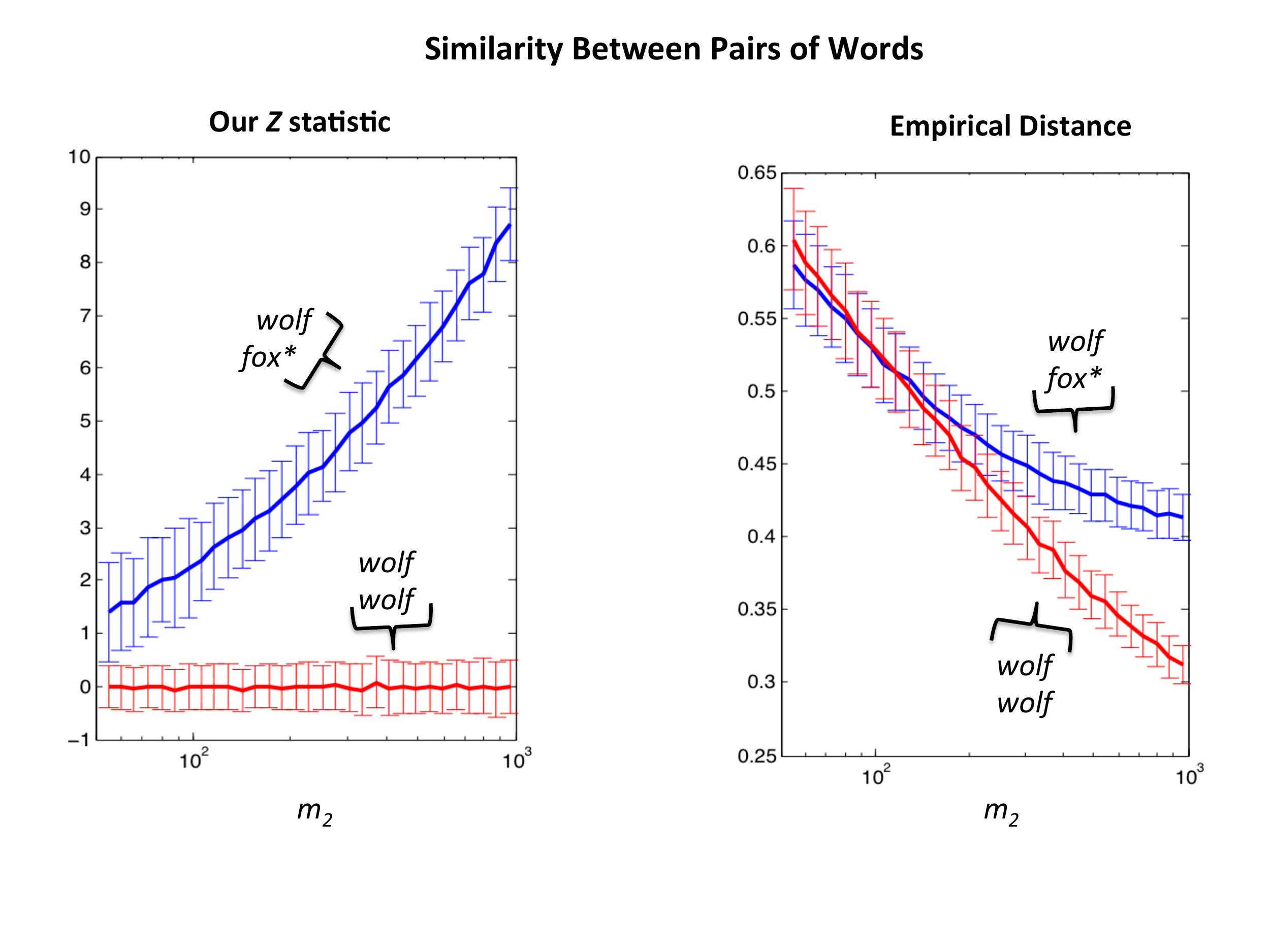}
\vspace{-1.5cm}
\caption{Illustration of how the empirical distance can be misleading: here, the empirical distance between the distributions of samples of bi-grams for ``wolf/wolf'' is indistinguishable from that for  the pair ``wolf/fox*'' over much of the range of $m_2$; nevertheless, our statistic clearly discerns that these are significantly different distributions.   Here, ``fox*" denotes the distribution of bi-grams whose first word is ``fox'', restricted to only the most common 100 bi-grams.  As in Figure~\ref{fig1}, $m_1=1,000$, and $m_2$ ranges from $50$ to $1,000$, with solid lines depicted the average of 200 trials, and error bars depicting one standard deviation.}\label{fig2}
\end{figure}

\begin{figure}\centering
\vspace{-1cm}
\includegraphics[width=400pt,height=240pt]{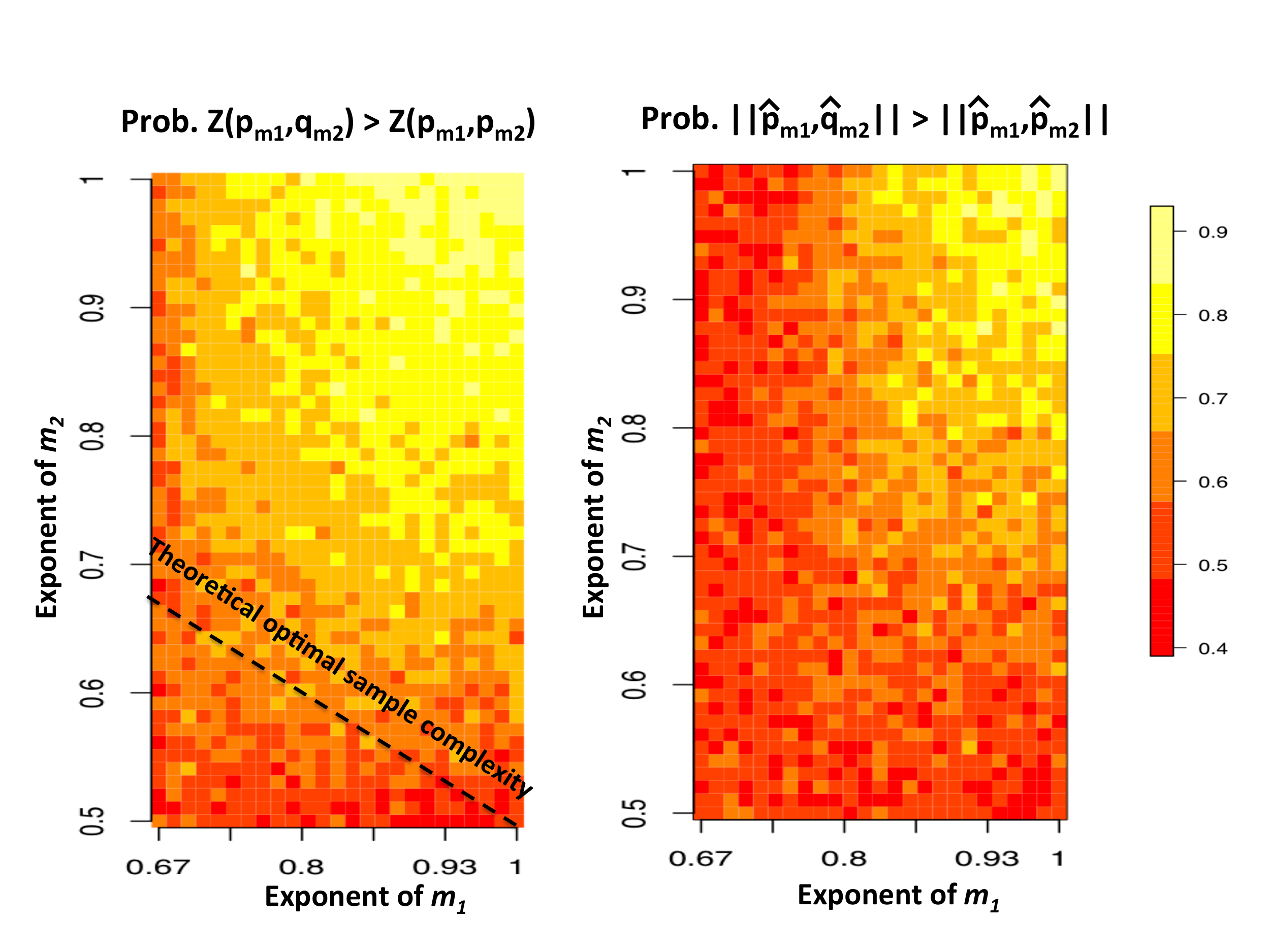}
\vspace{-.5cm}
\caption{A comparison of the $Z$ statistic versus the empirical distribution for distinguishing whether two samples of respective sizes $m_1,m_2,$ were both drawn from distribution $p:=\dUnif[5,000]$, versus one sample being drawn from $p$ and the other drawn from a distribution $q$ in which domain elements have probability $(1\pm \eps)/5000,$ for $\eps=1/4,$ and hence $||p-q||=1/4.$  The color signifies the fraction of 120 repetitions for which the statistic correctly distinguishes these cases, as $m_1$ varies between $n^{2/3}$ and $n$, and $m_2$ varies between $n^{1/2}$ and $n$.}
\label{fig3}
\end{figure}

\newpage

%\bibitem{1000genomes}
%G.R. Abecasis et al., An integrated map of genetic variation from 1,092 human genomes. {\it Nature}, %Vol. 491, pp. 56--65, 2012.

% Acknowledgments---Will not appear in anonymized version
%\acks{We thank a bunch of people.}

%\bibliography{yourbibfile}

\appendix

\section{Expectation and Variance Bounds}
\label{app:moments}

Before beginning the analysis of the algorithms we need bounds on the expectation and variance of
the different statistics used in the algorithms. Throughout this section, fix any set $A\subseteq [n]$, and let $X_i$ denote the number of occurrences of the $i$-th domain element in set $S_2$---a set of $\dPois(m_1)$ samples from distribution $p$, and analogously let $Y_i$ denote the number of occurrences of the $i$-th domain element in set $T_2$---a set of $\dPois(m_2)$ samples from distribution $q$.   Throughout this section, we bound the moments of the following statistics:
\begin{itemize}
\item $V_A=\sum_{i \in A}V_i = \sum_{i \in A} \left| \frac{X_i}{m_1}-\frac{Y_i}{m_2} \right|.$
\item $W_A=\sum_{i \in A}W_i = \sum_{i \in A}\left((m_2X_i - m_1 Y_i)^2 - (m_2^2X_i + m_1^2 Y_i)\right).$
\item $Z_A=\sum_{i \in A}Z_i = \sum_{i \in A}\frac{(m_2X_i - m_1 Y_i)^2 - (m_2^2X_i + m_1^2 Y_i)}{X_i+Y_i}.$
\end{itemize}

\subsection{Expectation and Variance of $V_A$}

\begin{lemma}For any fixed set $A\subseteq [n]$ %that is independent of samples in $S_2,T_2$,
\begin{equation}
\sum_{i\in A}|p_i-q_i|\leq \E[V_A]\leq \sum_{i\in A}|p_i-q_i|+\left(\frac{|A|}{m_1}+\frac{|A|}{m_2}\right)^{\frac{1}{2}}\leq \sum_{i\in A}|p_i-q_i|+\left(\frac{2|A|}{m_2}\right)^{\frac{1}{2}},	
\end{equation}
and
\begin{equation}
\Var[V_A]\leq \frac{1}{m_1}+\frac{1}{m_2}.
\end{equation}
\label{lm:variance_v}
\end{lemma}

\begin{proof} For the lower bound on the expectation, note that $\E\left[| \frac{X_i}{m_1}-\frac{Y_i}{m_2} |\right ]\geq \left|\E\left[ \frac{X_i}{m_1}-\frac{Y_i}{m_2}\right]\right| =|p_i-q_i|$. 

To prove the upper bound, observe that  
$$\E[V_i^2]=\frac{p_i}{m_1}+\frac{q_i}{m_2}+(p_i-q_i)^2.$$
By the Cauchy-Schwarz inequality,
\begin{eqnarray}
\E\left[\sum_{i\in A}V_i\right]\leq \sum_{i\in A}\E[V_i^2]^{\frac{1}{2}}&\leq & \sum_{i \in A}|p_i-q_i|+\sum_{i\in A}\left(\frac{p_i}{m_1}+\frac{q_i}{m_2}\right)^{\frac{1}{2}}\nonumber\\
&\leq &\sum_{i \in A}|p_i-q_i|+\left(\frac{|A|}{m_1}+\frac{|A|}{m_2}\right)^{\frac{1}{2}}.
\end{eqnarray}

Finally, $\Var[V_A]=\sum_{i \in A}(\E[V_i^2]-\E[V_i]^2)\leq  \frac{\sum_{i\in A}p_i}{m_1}+\frac{\sum_{i\in A}q_i}{m_2}\leq \frac{1}{m_1}+\frac{1}{m_2}$.
\end{proof}

\subsection{Expectation and Variance of $W_A$}

For $A\subseteq [n]$, define $W_A=\sum_{i\in A}W_i=\sum_{i \in A}(m_2X_i-m_1Y_i)^2-(m_2^2X_i+m_1^2Y_i)$.
Using the facts that $X_i\sim \dPois(m_1p_i)$ and $Y_i\sim \dPois(m_2q_i)$ and plugging in the expressions for the moments of Poissons, the following lemma follows immediately:

\begin{lemma}For any $A\subseteq [n]$, $W_A/(m_1^2m_2^2)$  is an unbiased estimate of $||p_A-q_A||_2^2$.  Namely,
\begin{equation}
\E[W_A]=m_1^2 m_2^2 \sum_{i \in A}(p_i - q_i)^2,
\label{eq:expectation_w}
\end{equation}
Moreover,
\begin{equation}
\Var[W_A]=2 m_1^2 m_2^2 \sum_{i\in A} z_i^2 + 4 m_1^3 m_2^3 \sum_{i\in A}z_i(p_i-q_i)^2,
\label{eq:variance_w}
\end{equation}
where $z_i=m_2p_i+m_1q_i.$
\label{lm:variance_w}
\end{lemma}

\subsection{Moments of $Z_A$} 

Recall that 
$$Z_i:=\frac{(m_2X_i-m_1Y_i)^2-(m_2^2X_i+m_1^2Y_i)}{X_i+Y_i},$$ and for $A\subseteq[n]$, $Z_A:=\sum_{i\in A}Z_i$. We show that if $p=q$, then $\E[\sum_{i\in A}Z_i]=0$, and otherwise, we give a lower bound on the expectation of the sum:

\begin{lemma} If $p=q$, then $\E[\sum_{i\in A}Z_i]=0$, and  otherwise, $\E[\sum_{i\in A}Z_i]\geq \frac{m_1^2m_2^2(\sum_{i\in A}|p_i-q_i|)^2}{4n+m_1+m_2}$.
\label{lm:expectation_Z}
\end{lemma}

\begin{proof}
Conditioned on the denominator, $$X_i \Big|X_i+Y_i=\sigma \sim \dBin\left(\sigma, \frac{m_1p_i}{m_1p_i+m_2q_i}\right).$$  
Set $\beta_i=\frac{m_1p_i}{m_1p_i+m_2q_i}$. Then using binomial moments we get,
\begin{eqnarray}
\E[(m_2X_i-m_1Y_i)^2|X_i+Y_i=\sigma]
%&=&m_2^2(S\beta_i(1-\beta_i)+S^2\beta_i^2)+m_1^2(S\beta_i(1-\beta_i)+S^2(1-\beta_i)^2)-2m_1m_2(S^2\beta_i(1-\beta_i)-S\beta_i(1-\beta_i))\nonumber\\
&=&\sigma\beta_i(1-\beta_i)(m_1+m_2)^2+\sigma^2(m_2\beta_i-m_1(1-\beta_i))^2\nonumber\\
&=&(m_1+m_2)^2\left(\sigma\beta_i(1-\beta_i)+\sigma^2\left(\frac{m_1}{m_1+m_2}-\beta_i\right)^2\right).
\end{eqnarray}
Similarly,
\begin{eqnarray}
\E[m_2^2X_i+m_1^2Y_i|X_i+Y_i=\sigma]&=&m_1^2 \sigma+(m_2^2-m_1^2)\E[X_i|X_i+Y_i=\sigma]\nonumber\\
&=&m_1^2 \sigma+(m_2^2-m_1^2)\sigma\beta_i\nonumber
\end{eqnarray}
Therefore, the conditional expectation of the numerator is 
\begin{eqnarray}
\E\left[m_2X_i-m_1Y_i)^2-(m_2^2X_i+m_1^2Y_i)\Big|X_i+Y_i=\sigma\right]&=&(m_1+m_2)^2\sigma(\sigma-1)\left(\frac{m_1}{m_1+m_2}-\beta_i\right)^2\nonumber\\
&=&\sigma(\sigma-1)\left(\frac{m_1m_2(q_i-p_i)}{m_1p_i+m_2q_i}\right)^2.
\label{eq:expectation_conditional}
\end{eqnarray}
This implies
$$\E\left[\sum_{i\in A}Z_i/m_1^2m_2^2\right]=\sum_{i\in A}\frac{(q_i-p_i)^2}{z_i}\left(1-\frac{1-e^{-z_i}}{z_i}\right),$$
where $z_i=m_1p_i+m_2q_i$.  This implies that the expectation of the sum is zero if $p=q$.  Let $g(z)=z/(1-\frac{1-e^{-z}}{z})$.  Now, using the fact that $g(z)\leq 2+z$ and the Cauchy-Schwarz inequality, the result follows.
\end{proof}

\begin{lemma} For $i\in [n]$ and $p=q$, $$\Var[Z_i]\leq 2 m_1^2m_2^2\Pr[X_i+Y_i>0], \text{ and hence} \Var[Z_A]=O(m_1^3m_2^2).$$

For $p_i\geq q_i$, $\Var[Z_i]\leq O(m_1^3m_2^2p_i)$, and for $p_i< q_i$
\begin{equation}
\Var[Z_i]\leq O(m_1^3m_2^2)\min\left\{ \frac{q_i^2}{p_i}, m_1q_i^2\right\}.
\label{eq:variance_z_II}
\end{equation}
\label{lm:variance_Z}
\end{lemma}

\begin{proof}The variance of $Z_i$ can be computed by using the formula for conditional variance. Define, $$G_i(\sigma):=\Var[(m_2X_i-m_1Y_i)^2-(m_2^2X_i+m_1^2Y_i)|X_i+Y_i=\sigma].$$ 
Let $\beta_i=\frac{m_1p_i}{m_1p_i+m_2q_i}$.
Using formulas for binomial moments the conditional variance 
\begin{eqnarray}
G_{i}(\sigma)&=&F_i(\sigma)+L_i(\sigma)\nonumber,
\label{eq:conditionalvariancenumerator}
\end{eqnarray}
where 
$$F_i(\sigma)=2 \beta_i^2 (1-\beta_i)^2 \sigma(\sigma-1) (m_1+m_2)^4, \quad L_i(\sigma)=4\beta_i(1-\beta_i)\sigma(\sigma-1)^2(m_1+m_2)^4 \left(\frac{m_1}{m_1+m_2}-\beta_i\right)^2.$$

For $p_i=q_i$, $\beta_i=\frac{m_1}{m_1+m_2}$ and $L_i(\sigma)=0$. Also, from the proof of Lemma \ref{lm:variance_Z} it can be seen that $\Var[\E[Z_i|X_i+Y_i=\sigma]]=0$, when $p_i=q_i$. Therefore, for $p_i=q_i$, $$\Var[Z_i]=\E[G_i(\sigma)/\sigma^2]=\E[F_i(\sigma)/\sigma^2]\leq 2 m_1^2m_2^2\Pr[X_i+Y_i>0].$$
Let $z_i=m_1p_i+m_2q_i$. Then $\Pr[X_i+Y_i>0]=1-e^{-z_i}\leq z_i$, and $\Var[Z_A]=\sum_{i\in A}\Var[Z_i]=O(m_1^3m_2^2)$.

To prove the bound in the case $p_i\ne q_i$, note that $F_i(\sigma)=0$, for $\sigma=0, 1$ and $F_i(\sigma)\leq 2 \beta_i^2 (1-\beta_i)^2 \sigma^2(m_1+m_2)^4$, for $\sigma\geq 2$. Therefore, 
\begin{eqnarray}
\E\left(\frac{F_i(\sigma)}{\sigma^2}\right)&\leq& 2(m_1+m_2)^4 \beta_i^2 (1-\beta_i)^2 \Pr[\sigma \geq 2]\nonumber\\
&\leq & 2(m_1m_2)^2(m_1+m_2)^4 \left\{\frac{p_i^2q_i^2(1-e^{-z_i}-z_ie^{-z_i})}{z_i^4}\right\}\nonumber\\
&\leq&O(m_1^6m_2^2)\left\{\frac{p_i^2q_i^2\min\{z_i, z_i^2\}}{z_i^4}\right\}.
\label{eq:F}
\end{eqnarray}
Now, for $p_i\geq q_i$, $z_i\geq \frac{m_1+m_2}{2}(p_i+q_i)$, and 
$$\E\left(\frac{F_i(\sigma)}{\sigma^2}\right)\leq O(m_1^6m_2^2)\left\{\frac{p_i^2q_i^2\min\{1, z_i\}}{z_i^3}\right\}\leq O(m_1^3m_2^2)\left\{\frac{p_i^2q_i^2}{(p_i+q_i)^3}\right\}\leq O(m_1^3m_2^2p_i).$$
The remaining terms in the variance can be bounded similarly, and for $p_i\geq q_i$, it follows that $\Var[Z_i]\leq O(m_1^3m_2^2p_i)$.

For the case $p_i< q_i$, use the bound $z_i\geq m_1p_i$ in (\ref{eq:F}) to get 
\begin{equation}
\E\left[\frac{F_i(\sigma)}{\sigma^2}\right]\leq O(m_1^3m_2^2)\min\left\{ \frac{q_i^2}{p_i}, m_1 q_i^2\right\}.
\label{eq:FL_I}
\end{equation}
%&\leq& 2(m_1m_2)^2(m_1+m_2)^4\sum_i \frac{p_i^2q_i^2}{z_i^4}\nonumber\\
%&\leq& 2(m_1m_2)^2\sum_i \frac{p_i^2q_i^2}{\min\{p_i, q_i\}^4}\nonumber\\
%&=& O(m_1^2m_2^2)\sum_i \frac{\max\{p_i, q_i\}^2}{\min\{p_i, q_i\}^2}=O(m_1^2m_2^4).
%&\leq&O(m_1^3m_2^2)\left(\sum_{i:p_i>q_i} \frac{p_i^2}{q_i}+\sum_{i:p_i\leq q_i}  \frac{q_i^2}{p_i}\right)\label{eq:conditionalvariance}

Similarly, $L_i(\sigma)=0$ for $\sigma=0, 1$ and $L_i(\sigma)\leq 4\beta_i(1-\beta_i)\sigma^3(m_1+m_2)^4 \left(\frac{m_1}{m_1+m_2}-\beta_i\right)^2$. Therefore, for the case $p_i< q_i$, using the bound $z_i^3\geq m_1^2m_2p_i^2q_i$, for $z_i\leq 1$, and $z_i^2\geq m_1m_2p_iq_i$, for $z_i\geq 1$ we get
\begin{eqnarray}
\E\left(\frac{L_i(\sigma)}{\sigma^2}\right)&\leq& 4(m_1+m_2)^4\beta_i(1-\beta_i) \left(\frac{m_1}{m_1+m_2}-\beta_i\right)^2\E[\sigma \pmb 1\{\sigma \geq 2\}]\nonumber\\
&=& 4m_1^3m_2^3(m_1+m_2)^2\frac{p_iq_i(p_i-q_i)^2z_i(1-e^{-z_i})}{z_i^4}\nonumber\\
&\leq & O(m_1^5m_2^3)\frac{p_iq_i(p_i-q_i)^2\min\{1, z_i\}}{z_i^3}\nonumber\\
&=&O(m_1^3m_2^2)\min\left\{\frac{q_i^2}{p_i}, m_1 q_i^2\right\}.
\label{eq:FL_II}
\end{eqnarray}
Finally, from Lemma \ref{lm:expectation_Z}  when $p_i<q_i$
\begin{eqnarray}
\Var[\E[Z_i|X_i+Y_i=\sigma]]&=&(m_1+m_2)^2 \Var[\sigma]\left(\frac{m_1}{m_1+m_2}-\beta_i\right)^2\nonumber\\
&=&m_1^4m_2^4\frac{(q_i-p_i)^4}{z_i^3}\nonumber\\
%&\leq& O(m_1^3m_2^2) \min\left\{\frac{(p_i-q_i)^4}{(p_i+q_i)^2\min\{p_i, q_i\}}, \frac{m_1(p_i-q_i)^4}{(p_i+q_i)\min\{p_i, q_i\}^2}\right\}\nonumber\\
&\leq &O(m_1^3m_2^2)\min\left\{\frac{q_i^2}{p_i}, m_1q_i^2\right\}.
\label{eq:FL_III}
\end{eqnarray}

Combining (\ref{eq:FL_I}), (\ref{eq:FL_II}), and (\ref{eq:FL_III}), the variance (\ref{eq:variance_z_II}) follows.
\end{proof}

For the analysis of the algorithms we also need bounds on the $s$-th moment of $Z_{A}$ corresponding to a set $A$ with the property that for all $i \in A$,  $p_i \leq 2b'$ and $q_i \leq 2b'$, where $b' = \frac{256 \log n}{m_2}$,  as define in Algorithm~\ref{frame1}.

\begin{lemma}For any $s\in \N$, and set $A \subset [n]$ such that for all $i \in A$,  $p_i \leq 2b'$ and $q_i \leq 2b'$,
$$\E[|Z_A-\E[Z_{A}]|^s]\leq \widetilde{O}_s(m_1^{2s}m_2),$$
\label{lm:Z_moments}
where $\widetilde O_s$ suppresses factor of $\log^{O(s)}n.$
\end{lemma}

\begin{proof}
Trivially, $|Z_i| \le 3 m_2^2 X_i + 3 m_1^2 Y_i.$
Since $\E[X_i^{s}]$ is a degree $s$ polynomial in $m_1p_i$, $\E[X_i^s]=O_s(\max\{m_1^sp_i^s, m_1p_i\})$. Similarly, for $\E[Y_i^s]=O_s(\max\{m_2^sq_i^s, m_2q_i\})$. Therefore, for $i \in A$,
\begin{eqnarray}\E[|Z_i|^s]=O_s(m_2^{2s}\E[X_i^s]+m_1^{2s} \E[Y_i^s])&=&\widetilde O_s(m_1^{2s}m_2\max\{p_i, q_i\}).
\label{eq:Z_H1}
\end{eqnarray}
Similarly, $\E[|Z_i|]^s=\widetilde O_s(m_1^{2s}m_2\max\{p_i, q_i\})$, and 
\begin{equation}
\E[|Z_A-\E[Z_A]|^s]\leq O_s\left(\sum_{i \in A}\E[|Z_i|^s]+\E[|Z_i|]^s\right)\leq \widetilde O_s(m_1^{2s}m_2).
\label{eq:Z_H2}
\end{equation}
Combining (\ref{eq:Z_H1}) and (\ref{eq:Z_H2}) yields the lemma.
\end{proof}

For the analysis of the algorithm in the extreme case, we will bounds on the $s$-th moment of $Z_{A}$ corresponding to a set $A$ with the property that for all $i \in A$,  $\frac{\varepsilon^{2/3}}{20m_2n^{1/3}} \leq p_i \leq 2b'$ and $q_i \leq 2b'$.  In this case, a more careful analysis gives a better bound on moments of $Z_A$.

\begin{lemma}For any $s\in \N$, and set $A \subset [n]$ such that for all $i \in A$,  $\frac{\varepsilon^{2/3}}{20m_2n^{1/3}} \leq p_i \leq 2b'$ and $q_i \leq 2b'$,
$$\E[|Z_A-\E[Z_{A}]|^s]\leq \tilde O\left(\frac{n^{s/3} m_1^{s}m_2^{s+1}}{\varepsilon^{2s/3}}\right),$$
\label{lm:ZA}
where $\widetilde O_s$ suppresses factor of $\log^{O(s)}n.$
\end{lemma}

\begin{proof} From the definition $Z_i$, 
$$|Z_i|\leq O\left(\frac{m_2^{2}X_i^2+m_1^{2} Y_i^2}{X_i+Y_i}\right).$$ Conditioned on $X_i+Y_i=\sigma$, $X_i\sim\dBin(\sigma, m_1p_i/z_i)$ and $Y_i\sim\dBin(\sigma, m_2q_i/z_i)$,
where $z_i=m_1p_i+m_2q_i$. Then, $\E[X_i]=\sigma m_2q_i/z_i:=x_i$, and for any $s\geq 1$, $$\E[X_i^s|X_i+Y_i=\sigma]=O(\max\{x_i, x_i^s\}).$$ Similarly, $$\E[Y_i^s|X_i+Y_i=\sigma]=O(\max\{y_i, y_i^s\}) \text { where } \E[Y_i]=\sigma m_2q_i/z_i:=y_i.$$ Therefore, for $\sigma>0$,
\begin{eqnarray}
\E[|Z_i|^s|X_i+Y_i=\sigma]& \leq & O_s\left(\max\left\{\frac{m_1^{2s}m_2^{2s}q_i^{2s} \sigma^s}{z_i^{2s}}, \frac{m_1^{2s}m_2q_i}{\sigma^{s-1}z_i^{s+1}}\right\}\right)\nonumber\\
&\leq & O_s\left(\max\left\{\frac{m_1^{2s}m_2^{2s}q_i^{2s} \sigma^s}{z_i^{2s}}, \frac{m_1^{2s}m_2q_i}{z_i^{s+1}}\right\}\right).
\label{eq:term}
\end{eqnarray}
Note that $\E[\sigma]=z_i$ and $\E[\sigma^s]=O_s(z_i^s)$ because $z_i\geq 1$ by assumption. Using $q_i\leq 2b'$ we get
\begin{equation}
O_s\left(\frac{m_1^{2s}m_2^{2s}q_i^{2s}}{z_i^s}\right)\leq O_s\left(\frac{m_1^{s}m_2^{2s}q_i^{2s}}{p_i^s}\right)\leq O_s\left(\frac{m_1^{s}m_2^{2s}b'^{2s-1}q_i}{p_i^s}\right)=\tilde O_s\left(\frac{m_1^{s}m_2q_i}{p_i^s}\right).
\label{eq:term1}
\end{equation}
Moreover, because $m_1p_i\geq 1$,
\begin{equation}
O_s\left(\frac{m_1^{2s}m_2q_i}{z_i^{s+1}}\right)\leq O_s\left(\frac{m_1^{s-1}m_2q_i}{p_i^{s+1}}\right)\leq O_s\left(\frac{m_1^{s}m_2q_i}{p_i^{s}}\right).
\label{eq:term2}
\end{equation}
Combining (\ref{eq:term1}) and (\ref{eq:term2}) with (\ref{eq:term}) and using $p_i\geq \frac{\varepsilon^{2/3}}{20m_2n^{1/3}}$ (since $i \in A$) gives
\begin{eqnarray}
\E[|Z_i|^s]\leq\tilde O_s\left(\frac{m_1^{s}m_2q_i}{p_i^s}\right)\leq \tilde O_s\left(\frac{n^{s/3} m_1^{s}m_2^{s+1}q_i}{\varepsilon^{2s/3}}\right).\nonumber
\end{eqnarray}
Similarly, it can be shown that $\E[|Z_i|]^s= \tilde O_s\left(\frac{n^{s/3} m_1^{s}m_2^{s+1}q_i}{\varepsilon^{2s/3}}\right)$, and 
\begin{equation}
\E[|Z_A-\E[Z_A]|^s]\leq O_s\left(\sum_{i \in A}\E[|Z_i|^s]+\E[|Z_i|]^s\right)\leq \widetilde O\left(\frac{n^{s/3} m_1^{s}m_2^{s+1}}{\varepsilon^{2s/3}}\right).
\label{eq:Z_H2}
\end{equation}
completing the proof of the lemma. 
\end{proof}

\section{Proof of Proposition \ref{ppn:testing_ub_I}}
\label{app:ppn:testing_ub_I}

We begin by establishing that, with high probability over the first set of samples, $S_1,T_1$, the sets $B,M,H$ successfully partition the elements in the ``heavy'', ``medium'', and ``light'' sets.  This proof follows from a union bound over Poisson tail bounds.   The proof of Proposition~\ref{ppn:testing_ub_I} will then proceed by arguing that, with high probability over the randomness of the second set of samples, $S_2,T_2$, the algorithm will be successful, provided that the sets $B,M,H,$ were a reasonable partition.   

\begin{definition}Let $b,b'$ be as defined in Algorithm~\ref{frame1}.  The set $B$ is said to be {\em faithful} if for all $i \in B$, $p_i> b/2$ or $q_i>b/2$. Similarly, $M$ is said to be {\em faithful} if for all $i \in M$, $b'/2\leq \max\{p_i, q_i\}\leq 2b$. Finally, $H$ is said to be {\em faithful} if $p_i< 2b'$ and $q_i< 2b'$, for all $i \in H$.
\label{def:faithfulness}
\end{definition}

\begin{lemma}\label{faithfulwhp}
With probability at $1-o(1/n)$ over the randomness in the samples $S_1,T_1$, the sets $B,M,$ and $H$ will be ``faithful''.
\end{lemma}
\begin{proof}
We leverage the following Chernoff style bound for Poisson distributions: for any $\lambda \le c,$ and $\delta \in (0,1),$ $$\Pr\left[|\dPois(\lambda)-\lambda| > \delta c \right] \le 2 e^{-\delta^2 c/3}.$$  Let $X^{S_1}_i$ denote the number of occurrences of $i$ in the $\dPois(m_1)$ samples, $S_1$, drawn from $p$, and $Y^{T_1}_i$ denote the number of occurrences of $i$ in the $\dPois(m_2)$ samples from $q$ that comprise $T_1$.   For any domain element $i$ with probability $p_i \ge b'/2$,  $$\Pr\left[|X^{S_1}_i - m_1p_i| \ge \frac{1}{2} m_1p_i \right] \le 2 e^{-\frac{1}{4 \cdot 3} m_1p_i} \le 2 e^{- 20 \log n} = o(1/n^2).$$   Similarly,  for any domain element $i$ with probability $q_i \ge b'/2$,  $$\Pr\left[|Y^{T_1}_i - m_2q_i| \ge \frac{1}{2} m_2q_i \right] \le 2 e^{-\frac{1}{4 \cdot 3} m_2q_i} \le 2 e^{- 20 \log n} = o(1/n^2).$$  
 So far, this ensures that common elements do not occur too infrequently.  To ensure that none of the rare elements occur too frequently, note that the same bound implies that for any domain element $i$ with probability $p_i \le b'/2$,  $$\Pr\left[X^{S_1}_i \ge b' m_1\right] \le \Pr\left[|X^{S_1}_i - m_1p_i| \ge  b' m_1 /2 \right] \le 2 e^{-b'  m_1/6} \le 2 e^{- 20 \log n} = o(1/n^2).$$   Analogously for any domain element $i$ with probability $q_i \le b'/2$,  $$\Pr\left[Y^{T_1}_i \ge b' m_2 \right] \le \Pr\left[|Y^{S_1}_i - m_2q_i| \ge  b' m_2/2 \right] \le 2 e^{-b' m_2/6} \le 2 e^{- 20 \log n} = o(1/n^2).$$
 
Note that if, for all domain elements $i$ with $p_i \ge b'/2$,  $|X^{S_1}_i - m_1p_i| < \frac{1}{2} m_1p_i,$ and for all elements $i$  with $p_i \le b'/2$, $X^{S_1}_i \le b' m_1$, and the analogous statements hold for $q_i$ and $Y^{T_1}_i$, then the sets $B,M,$ and $H$ will all be ``faithful.   By our above bounds, and a union bound over the $n$ elements, with probability at least $1-o(1/n)$ this occurs.
\end{proof}

We  now prove the correctness of Algorithm (\ref{frame1}) by establishing that in the case that $p=q$, the algorithm will output ACCEPT with probability at least $2/3$, and in the case that $||p-q||_1 \ge \eps$ the algorithm will output REJECT with probability at least $2/3.$ The analysis of these two cases is split into Lemmas~\ref{lemma:pq} and~\ref{lemma:pnq}.  Together with Lemma~\ref{faithfulwhp}, this establishes Proposition~\ref{ppn:testing_ub_I}:

\subsection{$||p-q||_1=0$} We analyze the statistics of the algorithm in the case that $p=q$, with respect to the randomness in the samples $S_2,T_2$ under the assumption that the sets $B,M,H$ are faithful. 

\begin{lemma}\label{lemma:pq}
Given that the sets $B,M,$ and $H$ are ``faithful'' and that $p=q$, then with high probability over the randomness in $S_2,T_2$, Algorithm~\ref{frame1} will output ACCEPT.
\end{lemma}
\begin{proof}
\subsubsection{The statistic $V_B$:}
\label{sssec:V_B_I}
By Lemma~\ref{lm:variance_v}, $$\E[V_B] \le \left( \frac{2|B|}{m_2}\right)^{1/2} +  \sum_{i \in B} |p_i - q_i| = \left( \frac{2|B|}{m_2}\right)^{1/2}.$$  From our definition of ``faithful'', every element of $i \in B$ must have $p_i + q_i \ge b/2 = \frac{128 \log n}{\eps^2 m_2},$ hence $|B| \le \frac{2\eps^2 m_2}{128 \log n} < \frac{\eps^2 m_2}{64 \log n},$ and $$\E[V_B] \le \left( \frac{2|B|}{m_2}\right)^{1/2} \le \eps \frac{\sqrt{2}}{8 \sqrt{\log n}} < \eps/8, \text{for $n >2$}.$$    From Lemma~\ref{lm:variance_v}, $\Var[V_B] \le \frac{1}{m_1}+\frac{1}{m_2} \le \frac{\eps^2}{\sqrt{n}} = o(\eps^2).$  Hence by Chebyshev's inequality, $\Pr[V_B > \eps/6] \le o(1),$ and hence the first check of Algorithm~\ref{frame1} will pass.

\iffalse
Note that $|B|\leq \frac{1}{b}\leq m_2\varepsilon^{2}/288$. Moreover, by Lemma \ref{lm:variance_w}, $\E[V_B|B]\leq \left(\frac{2|B|}{m_2}\right)^{\frac{1}{2}}\leq \varepsilon/12$ and $\Var[V_B|B]=O\left(\frac{m_1^{1/2}\varepsilon^2}{n}, \frac{\varepsilon^2}{n^{1/2}}\right)$,  since $B$ is independent of the samples $\mathcal S_2$ and $\mathcal T_2$. Therefore, 
\begin{eqnarray}
\Pr[V_B>\varepsilon/6|B) &\leq& \Pr[V_B-\E[V_B)>\varepsilon/12|B)\nonumber\\
&\leq&\frac{144 \Var[V_B|B)}{\varepsilon^2}\leq \frac{1}{10}.
\end{eqnarray}
Now, taking expectation with respect to $B$ gives $\Pr[V_{ B}>\varepsilon/6]\leq 1/10$.
\fi
\subsubsection{The statistic $W_M$:}
\label{sssec:W_M_I}

From Lemma~\ref{lm:variance_w}, $\E[W_M]=m_1^2m_2^2 \sum_{i \in M} (p_i-q_i)^2 = 0.$  Additionally, $$\Var[W_M]=2m_1^2m_2^2 \sum_{i \in M}(m_2 p_i + m_1 q_i)^2 \le 2m_1^2m_2^2 \cdot \max_i \{m_2p_i+m_1q_i\}\sum_i(m_2 p_i +m_1 q_i).$$  From the fact that $M$ is faithful, $\max_i\{m_2p_i+m_1q_i\} \le O(\frac{m_1 \log n}{m_2 \eps^2}),$ and hence we conclude that $\Var[W_M] = O(\frac{m_1^4 m_2 \log n}{\eps^2}).$

 By Chebyshev's inequality, and the assumption that $\eps > 1/n^{1/12}$,  $$\Pr\left[W_M \ge \frac{\eps^2 m_1^2 m_2 \log n}{2}\right] = o(1),$$ and hence the second check of Algorithm~\ref{frame1} will pass.

\iffalse
Recall that $M=\left\{i \in [n]: b'\leq \max\{\hat p_i, \hat q_i\} \leq b \right\}$, where $b=\frac{288 \log n}{m_2\varepsilon^2}$ and $b'=\frac{288\log n}{m_2}$. From the proof of Lemma \ref{lm:variance_w} and the independence of $M$ and $\mathcal S_2$ and $\mathcal T_2$, it follows that  when $p=q$, $\E[W_M|M]=0$, and if $M$ is faithful
$$\Var[W_M|M]=2m_1^2m_2^2\sum_{i\in M}z_i^2=O(\varepsilon^{-2}m_1^4m_2\log n),$$
since $z_i=m_2q_i+m_1p_i\leq O(\frac{m_1\log n}{m_2\varepsilon^2})$. Therefore, $\Var[W_M]=\E[\Var[W_M|M)]=O(\varepsilon^{-2}m_1^4m_2\log n)$, when $M$ is faithful. This implies for some absolute constant $C$,
\begin{eqnarray}
\Pr\left[W_M>\frac{\varepsilon^2 m_1^2 m_2\log n}{2}\right]&\leq&  \Pr\left[\frac{W_M}{\sqrt{\Var[W_M)}}>  C\varepsilon^3\sqrt{m_2\log n}\right]+1/100\nonumber\\
&=& O(1/\log n)+1/100,
\nonumber,
\end{eqnarray}
using $\varepsilon \geq n^{-\frac{1}{12}}\geq n^{-\frac{1}{8}}$. This implies $\Pr\left[W_M>\frac{\varepsilon^2 m_1^2 m_2\log n}{2}\right]\leq 1/20$.
\fi
\subsubsection{The statistic $Z_H$:}
\label{sssec:Z_H_I}

By Lemma \ref{lm:expectation_Z}, $\E[Z_H] = 0,$ and by Lemma \ref{lm:variance_Z}, $\Var[Z_H]=O(m_1^3m_2^2)$. Therefore, by Chebyshev inequality $\Pr[Z_H\geq C_\gamma m_1^{3/2}m_2]\leq O(\frac{1}{C_\gamma^2}),$ which can be made arbitrarily small for a sufficiently large constant $C_{\gamma},$ and hence the third check of Algorithm~\ref{frame1} will pass.
 \end{proof}

\subsection{$||p-q||_1\geq \varepsilon$}  
\label{sec:l1greater}
 
We now consider the execution of the algorithm when $||p-q||_1\geq \varepsilon$.
 
\begin{lemma}\label{lemma:pnq}
Given that the sets $B,M,$ and $H$ are ``faithful'' and that $||p-q||_1\geq \varepsilon$, then with high probability over the randomness in $S_2,T_2$, Algorithm~\ref{frame1} will output REJECT.
\end{lemma}
\begin{proof}
The proof proceeds by considering the following three cases, at least one of which holds: 1) $\sum_{i \in B}|p_i-q_i|\geq \varepsilon/3$, 2) $\sum_{i \in M}|p_i-q_i|\geq \varepsilon/3$, and 3) $\sum_{i \in H}|p_i-q_i|\geq \varepsilon/3$.

\subsubsection{$\sum_{i \in B}|p_i-q_i|\geq \varepsilon/3$}
\label{sssec:V_B_II}

By Lemma \ref{lm:variance_v}, $\E[V_B]\geq \sum_{i\in B}|p_i-q_i| \ge \eps/3$ and $\Var[V_B] \le \frac{1}{m_1} + \frac{1}{m_2} \le 2/\sqrt{n}$  Therefore by Chebyshev's inequality, $\Pr[V_B<\varepsilon/6]= o(1),$
and hence the algorithm will output REJECT with high probability.

\subsubsection{$\sum_{i \in M}|p_i-q_i|\geq \varepsilon/3$}
\label{sssec:W_M_II}
From Lemma \ref{lm:variance_w}, $\E[W_M] = m_1^2 m_2^2 \sum_{i \in M}(p_i-q_i)^2.$   From the definition of ``faithful'', it follows that $|M| \le 2 \frac{m_2}{128 \log n}$, and hence by Cauchy-Schwarz, $$(m_1^2 m_2^2)\sum_{i \in M} (p_i-q_i)^2 \ge (m_1^2 m_2^2)\frac{\left(\sum_{i \in M}|p_i - q_i| \right)^2}{|M|} \ge (m_1^2 m_2^2) \frac{128 \eps^2 \log n}{18 m_2} \ge 7 \eps^2 m_1^2 m_2 \log n.$$   

Furthermore, from Lemma \ref{lm:variance_w},  $$\Var[W_M]\leq 2 m_1^2 m_2^2 \sum_{i \in M} z_i^2 + 4 m_1^3 m_2^3 \sum_{i\in M}z_i(p_i-q_i)^2,$$ where $z_i=m_1q_i+m_2p_i.$  As in the proof of Lemma~\ref{lemma:pq}, the first term is $O(\frac{m_1^4 m_2 \log n}{\eps^2})$.  For the second term, noting that $\sum_i z_i \le m_1+m_2$, and $(p_i-q_i)^2 \le O(\frac{\log^2 n}{\eps^4 m_2^2}),$ we get the bound of $O(\frac{m_1^4 m_2 \log n}{\eps^4}).$

By Chebyshev's inequality and the assumption that $\eps>1/n^{1/12}$, with probability $1-o(1)$, $W_M > \eps^2 m_1^2 m_2 \log n$, and the algorithm will output REJECT.

\iffalse
 \nonumber\\
&\leq & O(\varepsilon^{-4}m_1^4m_2\log ^2n),
\end{eqnarray}
since $z_i=m_1q_i+m_2p_i=O(\frac{m_1\log n}{m_2\varepsilon^2})$ and $(p_i-q_i]\leq p_i+q_i\leq 2b\leq \log n/m_2\varepsilon^2$, whenever $M$ is faithful.
Therefore, there exists some absolute constant $C$ such that
\begin{eqnarray}
\Pr\left[W_M<\frac{\varepsilon^2 \ m_1^2 m_2\log n}{2}, M \text{ is faithful }|M\right]&\leq &\Pr\left[\frac{W_M-\E[W_M|M)}{\sqrt{\Var[V_M|M)}}<-C\varepsilon^4\sqrt m_2, M \text{ is faithful }|M\right]\nonumber\\
&\leq& O(1),\nonumber
\end{eqnarray}
using $\varepsilon\geq n^{-\frac{1}{12}}$. 

Since $M$ is faithful with high probability (Lemma \ref{lm:BMH}),  $\Pr[W_{M}<\frac{\varepsilon^2 m_1^2 m_2\log n}{2}]\leq 1/20$.

\fi

\subsubsection{$\sum_{i \in H}|p_i-q_i|\geq \varepsilon/3$} 
\label{sssec:Z_H_II}

From Lemma \ref{lm:expectation_Z}, $\E[Z_H] \ge \Omega(\frac{m_1^2m_2^2\varepsilon^2}{n}).$  Using the assumption in the statement of Proposition~\ref{ppn:testing_ub_I} that $m_2 = \Omega(\frac{n}{\eps^2 \sqrt{m_1}}),$ we conclude that $$\E[Z_H] = \Omega(m_1^{3/2} m_2).$$  Using the moment bounds from Lemma~\ref{lm:Z_moments} and the definition of ``faithful'', for any integer $s>0$, $\E[|Z_H - \E[Z_H]|^s] \le \tilde{O_s}(m_1^{2s}m_2).$  By Markov's inequality, 

\begin{eqnarray}
\Pr[Z_{ H}\leq C_\gamma m_1^{3/2}m_2] &\leq&\Pr\left[|Z_{ H}-\E[Z_{ H}]|\geq  \Omega(m_1^{3/2}m_2)\right]\nonumber\\
& = &\Pr\left[|Z_{ H}-\E[Z_{ H}]|^s\geq  \Omega(m_1^{3s/2}m^s_2)\right]\nonumber\\
%&\leq & \frac{\E[|Z_{ H}-\E[Z_{ H})|^s)}{\tilde C_\gamma ^sm_1^{3s/2}m_2^s}\nonumber\\
&\leq &\widetilde O_s\left(\frac{m_1^{2s}m_2}{m_1^{3s/2}m_2^s}\right) =  \widetilde O_s\left(\frac{m_1^{\frac{s}{2}}}{m_2^{s-1}}\right).\nonumber
\end{eqnarray}

As long as $\frac{m_1}{m^2_2} \le 1/n^c$ for some positive constant $c$,  there will be some integer $s_c,$ dependent on $c$ for which this probability is $o(1)$.  Note that the stipulation in the proposition statement, that $m_1 = O\left( (n/\eps^2)^{1-\gamma}\right),$ for some constant $\gamma>0$, ensures that $\frac{m_1}{m_2^2} = O(1/n^{-2\gamma}),$ and hence the algorithm will output REJECT with probability $1-o(1)$ in this case.
\iffalse

As this expression is $o(1)$

&\leq &\widetilde O_s\left(\left\{\frac{m_1}{(n/\varepsilon^2)^{\lambda(s)}}\right\}^{s-\frac{1}{2}}\right)\nonumber\\
&\leq &\widetilde O_s\left(\left(\frac{(n/\varepsilon^2)^{1-\gamma}}{(n/\varepsilon^2)^{\lambda(s)}}\right)^{s-\frac{1}{2}}\right)=o(1),
\label{eq:ZH2}

by choosing $s$ large enough so that $\lambda(s):=\frac{s-1}{s-\frac{1}{2}}>1-\gamma$, since $m_1=O((n/\varepsilon^2)^{1-\gamma})$, for some $0<\gamma\leq 1/3$. 

Since  $H$ is faithful with probability at least 99/100, it follows that $\Pr[Z_H<C_\gamma m_1^{3/2}m_2]\leq 1/20$.\fi
\end{proof}

\section{Proof of Proposition \ref{ppn:testing_ub_II}}
\label{app:ppn:testing_ub_II}

In this section we prove Proposition~\ref{ppn:testing_ub_II}, showing that Algorithm~\ref{frame2} performs as claimed in the {\it extreme case} where $m_1\approx n$. The algorithm is a slight modification of Algorithm (\ref{frame1}), tailored to handle the imbalance between the sample sizes from $p$ and $q$.  We prove that this algorithm works whenever $m_1=\Omega(\left(n/\varepsilon^2\right)^{8/9+\gamma})$ for some $\gamma>0$, and overlaps with the regime of parameters for which the non-extreme algorithm, Algorithm~\ref{frame1}, will succeed.

We begin the proof of the above proposition by considering the statistic $R_H$. 

\begin{observation}Define $R_A=\sum_{i \in A}\frac{\pmb 1\{Y_i=2\}}{X_i+1}$, for $A\subseteq [n]$. Then
\begin{equation}
\E[R_A]=\sum_{i=1}^n\frac{m_2^2 q_i^2 \left(1-e^{-m_1p_i}\right) e^{-m_2 q_i}}{2m_1p_i}.
\label{eq:expectation_R}
\end{equation}
\label{obs:R}
\end{observation}

\begin{proof} Since $X_i\sim \dPois(m_1p_i)$, $\E[\frac{1}{X_i+1}]=\frac{1-e^{-m_1p_i}}{m_1p_i}$. Also, $Y_i\sim \dPois(m_2q_i)$ implies $\Pr[Y_i=2]=\frac{(m_2q_i)^2}{2}e^{-m_2q_i}$. The expectation of $R_A$ now follows from linearity of expectation and the independence of $X_i$ and $Y_i$.
\end{proof}

As mentioned before, in the extreme case the statistic $Z_A$ can incur a variance of $O(n^4)$, which is at the threshold of what can be tolerated. The statistic $R_A$ is tailored to deal with these cases. This is formalized in the following lemmas: whenever the variance of $Z_A$ is at least the tolerance threshold $\Omega(m_1^3m_2^2)$, the expected values of $R_A$ in the case $p=q$ is well separated from the likely values of $R_A$ in case $||p-q||_1> \varepsilon$.

\begin{lemma}If $p=q$, $\E[R_A]\leq \frac{m_2^2}{2 m_1}$.  If $p\ne q$ and $\max_{i \in A} q_i \le \frac{10}{m_2}$ and $\Var[Z_A]=\Omega(m_1^3m_2^2)$,  then  $\E[R_A]\geq \Omega(m_2^2/m_1)$.
\label{lm:R}
\end{lemma}

\begin{proof} If $p=q$, then
$$\E[R_A]=\frac{m_2^2}{2 m_1} \sum_{i\in A}\frac{q_i^2 \left(1-e^{-m_1p_i}\right) e^{-m_2 q_i}}{2p_i}\leq \frac{m_2^2}{2 m_1} \sum_{i\in A}\frac{q_i^2}{2p_i}\leq \frac{m_2^2}{2 m_1}.$$

Now, suppose $p\ne q$. Let $$A_0:=\{i \in A: m_1p_i\geq 1/2\}.$$ Note that $\Var[Z_A]\ge \Omega(m_1^{3}m_2^2)$ implies that either $\sum_{i\in A_0}\frac{q_i^2}{p_i}\geq C$ or $m_1\sum_{i\in A\setminus A_0} q_i^2\geq C$ for some constant $C$ (since by Lemma \ref{lm:variance_Z}, $\Var[Z_A]\leq O(m_1^3m_2^2)\sum_{i \in A}\min\left\{ \frac{q_i^2}{p_i}, m_1q_i^2\right\}$). We consider the two cases separately:

\begin{description}
\item[1 ] Suppose $\sum_{i\in A_0}\frac{q_i^2}{p_i}\geq C$. Since $q_i \le 10/m_2$ for all $i \in A$, it holds that for $i \in A_0, e^{-m_2q_i}\geq e^{-10}$. Moreover, $i \in A_0$ implies $1-e^{-m_1p_i}\geq 1-e^{-1/2}$. Therefore, 
$$\sum_{i\in A_0}\frac{m_2^2q_i^2 \left(1-e^{-m_1p_i}\right) e^{-m_2 q_i}}{2m_1p_i} \ge \frac{e^{-12}m_2^2}{m_1} \sum_{i\in A_0}\frac{q_i^2}{p_i} \geq \frac{C \cdot e^{-12}m_2^2}{m_1} .$$

\item[2] Suppose $m_1\sum_{i\in A\setminus A_0} q_i^2\geq C$. Using the inequality $1-e^{-x}\geq x-x^2/2$,
\begin{eqnarray}
\sum_{i\in A\setminus A_0}\frac{m_2^2q_i^2 \left(1-e^{-m_1p_i}\right) e^{-m_2 q_i}}{2m_1p_i}&\geq& \frac{e^{-10}m_2^2}{2 m_1} \sum_{i\in A\setminus A_0}\frac{q_i^2\left(m_1p_i-\frac{m_1^2p_i^2}{2}\right)}{p_i}\nonumber\\
&=&\frac{e^{-10}m_2^2}{2 m_1} \sum_{i\in A\setminus A_0}(m_1q_i^2-m_1^2q_i^2p_i/2)\nonumber\\
&\ge&\frac{e^{-10}m_2^2}{2} \sum_{i\in A\setminus A_0}(q_i^2-q_i^2/4)\nonumber\\
&=&\frac{e^{-10}m_2^2}{2} \sum_{i\in A\setminus A_0}3q_i^2/4 \ge \frac{C \cdot 3 e^{-10}m_2^2}{8},\nonumber
\end{eqnarray}
where the second to last inequality uses that assumption that $m_1 p_i < 1/2$ for $i \in A \setminus A_0$. 
%Let $c_1$ and $c_2$ be constants such that $\E[R]\leq c_1m_2^2/m_1$, when $p=q$, and $\E[R]\geq c_2 m_2^2/m_1$, when $p\ne q$.
\end{description}
Combining the above cases it follows that $\E[R_A]\geq \Omega(m_2^2/m_1)$.
\end{proof}

From the proof of the above lemma it is clear that we can choose some absolute constant $K$ such that whenever if $p\ne q$ and 
\begin{equation}
\max_{i \in A}|q_i|\leq 10/m_2, \quad \Var[Z_A]\geq Km_1^3m_2^2,
\label{eq:K}
\end{equation}
then $\E[R_A]\geq 11m_2^2/2m_1$. Hereafter, fix this constant $K$.

\subsection{$p=q$} Suppose, $m_1=\Omega((n/\varepsilon^2)^{8/9+\gamma})$ for some $\gamma >0$.  We analyze the statistics in Algorithm \ref{frame2} in the case that $p=q$, with respect to the randomness in the samples $S_2,T_2$ under the assumption that the sets $B,M,H$ are faithful.

\begin{lemma}\label{lemma:pnqext}
Given that the sets $B,M,$ and $H$ are ``faithful'' and that $p=q$, then with high probability over the randomness in $S_2,T_2$, Algorithm~\ref{frame2} will output ACCEPT.
\end{lemma}

\begin{proof} From calculations identical to those in case  \ref{sssec:V_B_I}, \ref{sssec:W_M_I}, it follows that  
$$\Pr[V_{ B}\geq \varepsilon/6]\leq \frac{1}{100},\quad \Pr[W_{ M}\geq  \frac{\varepsilon^2 m_1^2 m_2\log n}{2}]\leq \frac{1}{100}, \quad \Pr[Z_{ H}\geq C_2 m_1^{3/2}m_2]\leq \frac{1}{100},$$
when $p=q$. Therefore, the unknown distributions will pass the checks in Algorithm \ref{frame2} that correspond to the statistics $V_B$, $W_M$, and $Z_H$.

It remains to verify the additional two checks in Algorithm \ref{frame2}.

\subsubsection{Check (1) in Algorithm \ref{frame2}} 
\label{sec:unbalanced}
To show that the first check in Algorithm \ref{frame2} passes, we will show that when $p=q$,  
$$\Pr\left[ \text{ there exists } i\in[n] \text{ such  that } Y_i\geq 3 \text{ and } X_i\leq \frac{m_1\varepsilon^{2/3}}{10m_2n^{1/3}}\right] < 1/50.$$ Denote $\lambda= \frac{m_1\varepsilon^{2/3}}{10m_2n^{1/3}} =
\Omega\left(\frac{m_1^{3/2}\varepsilon^{8/3}}{n^{4/3}}\right)=\Omega(n^\gamma)$ for some constant $\gamma > 0$, since by assumption, $m_1=\Omega((n/\varepsilon^2)^{8/9+\gamma})$ for some $\gamma >0$.

If $p_i> \frac{2\lambda}{m_1}$. Then $\Pr\left[X_i\leq \lambda\right]\leq \Pr[\dPois(2\lambda)\leq \lambda] = o(1/n^2)$.   On the other hand, if $p_i=q_i \le \frac{2 \lambda}{m_1},$ then $$\Pr[Y_i \ge 3] \le \Pr\left[\dPois\left( \frac{2 \lambda m_2}{m_1}\right) \ge 3\right] = \Pr\left[\dPois\left( \frac{2 \eps^{2/3}}{10 n^{1/3}}\right) \ge 3\right] < \frac{1}{100 n}.$$ Hence by a union bound over all $i \in [n]$, check (1) in Algorithm \ref{frame2} passes.

\iffalse
 as $X_i\sim\dPois(m_1p_i)$, and $m_1p_i>2\lambda$. Therefore,
\begin{equation}
\Pr[Y_i\geq 3 \text{ and }X_i\leq \lambda]\leq \Pr\left[X_i\leq \lambda\right]\leq \Pr[\dPois(2\lambda)\leq \lambda]\leq 2e^{-\frac{\lambda}{6}}= 2e^{-\Omega(n^{\gamma})}.
\label{eq:plarge}
\end{equation}

\item[{\it Case} 2:] If $p_i\leq \frac{2\lambda}{m_1}$, then $q_i=p_i\leq \frac{2\lambda}{m_1} \leq \frac{2}{m_2n^{1/3}}$, since w.l.o.g. $\varepsilon \leq 1$. As $Y_i\sim\dPois(m_2q_i)$, and $m_2q_i \leq \frac{2}{n^{1/3}}$,
\begin{equation}
\Pr[Y_i\geq 3 \text{ and }X_i\leq \lambda]\leq \Pr[Y_i\geq 3]\leq \Pr[\dPois(2/n^{1/3}]\geq 3]\leq O(1/n).
\label{eq:psmall}
\end{equation}
\end{description}\fi

\subsubsection{The statistic $R$} Recall that $H=[n]\backslash (B\cup M)$, where $B$ and $M$ are defined in (\ref{frame1}). Note that by Lemma \ref{lm:R}, when $p=q,$ $$\E[R_H]\leq \frac{m_2^2}{2m_1}.$$  Recall that $m_2^2/m_1 \ge 1$, and the second criteria for Algorithm~\ref{frame2} rejecting is $R_H > C m_2^2/m_1,$ for a large constant $C$.  Since $R_H$ is a sum of independent random variables, each of which is in the range $(0,1)$, a standard Chernoff bound applies, yielding that the probability the algorithm rejects due to this $R_H$ is at most $1/100$.
\iffalse
From (\ref{eq:K}), . By a Chernoff bound applied to the sum $R_H,$ (since each term in the sum is bounded in 0 and 1),
\begin{eqnarray}
\Pr[R_H> 3m_2^2/m_1]\leq \Pr[R_H-\E[R_H)> 2.5 m_2^2/m_1)&\leq & e^{-O\left(m_2^4/m_1^2\E[R_H)\right)}\nonumber\\
&\leq &e^{-O(m_2^2/m_1)}\nonumber\\
&=&e^{-O(\varepsilon^{-2}n^2/m_1^2)}=o(1),
\end{eqnarray}
since $m_1\leq n$. \fi
\end{proof}

\subsection{$||p-q||_1\geq \varepsilon$}

\begin{lemma}\label{lemma:pnq}
Given that the sets $B,M,$ and $H$ are ``faithful'' and that $||p-q||_1\geq \varepsilon$, then with high probability over the randomness in $S_2,T_2$, Algorithm~\ref{frame1} will output REJECT.
\end{lemma}
\begin{proof}
The proof proceeds by considering the following three cases, at least one of which holds: 1) $\sum_{i \in B}|p_i-q_i|\geq \varepsilon/3$, 2) $\sum_{i \in M}|p_i-q_i|\geq \varepsilon/3$, and 3) $\sum_{i \in H}|p_i-q_i|\geq \varepsilon/3$. Now, if either $\sum_{i \in B}|p_i-q_i|\geq \varepsilon/3$ or $\sum_{i \in M}|p_i-q_i|\geq \varepsilon/3$, then from calculations identical to those in Sections \ref{sssec:V_B_II}, \ref{sssec:W_M_II}
it follows that the algorithm will output REJECT.

Therefore, assume that $\sum_{i \in H}|p_i-q_i|\geq \varepsilon/3$. We begin the proof with the following observation:

\begin{observation}Suppose there exists $j\in [n]$ such that $q_j\geq \frac{10}{m_2}$ and $p_j\leq \frac{\varepsilon^{2/3}}{20m_2n^{1/3}}$, then 
\begin{equation}
\Pr\left[ \exists i\in [n] s.t. Y_i\ge 3 \text{ and }X_i\le \frac{m_1\varepsilon^{2/3}}{10m_2n^{1/3}}\right]\ge \frac{9}{10},
\label{ob1}
\end{equation}
that is, Algorithm \ref{frame2} fails the first check and REJECTS.
\label{ob1}
\end{observation}

\begin{proof}
Given $j$ with  $q_j\geq \frac{10}{m_2}$ and $p_j\leq \frac{\varepsilon^{2/3}}{20m_2n^{1/3}}$, $\Pr[Y_j \ge 3] > 0.99$, and $\Pr\left[X_j <  \frac{m_1\varepsilon^{2/3}}{10m_2n^{1/3}}\right]> 1-o(1).$
\end{proof}

Given this observation, we may continue under the assumption that for all $i\in [n]$ such that $q_i\geq \frac{10}{m_2}$,  $p_i\geq  \frac{\varepsilon^{2/3}}{20m_2n^{1/3}}$. Now,  define $$S_0:=\{i\in [n]: q_i\leq 10/m_2\},$$ and consider the following cases:

\begin{description}

\item[{\it Case} 1]$\sum_{i \in S_0}|p_i-q_i|\geq \varepsilon/6$.  To begin with suppose that  $\Var[Z_{S_0}]\leq K m_1^3m_2^2$, with $K$ as defined in (\ref{eq:K}). Then by Chebyshev's inequality $\Pr[Z_{H}\leq C_2 m_1^{3/2}m_2]\leq  \frac{1}{20}$ (since $\E[Z_{S_0}]\geq \Omega(m_1^{3/2}m_2)$ by  Lemma \ref{lm:expectation_Z}). Otherwise, $\Var[Z_{S_0}]\geq Km_1^{3}m_2^2$, in which case, by Lemma \ref{lm:R}, $\E[R_{S_0}]\geq \frac{11m_2^2}{2m_1}$; since $R_{H} \ge R_{S_0}$ is a sum of independent random variables, with values between $0$ and $1$, a Chernoff bound yields that with probability at least $0.99,$ $R_H$ will exceed the threshold and the second check of Algorithm~\ref{frame2} will fail.
\iffalse
\begin{eqnarray}
\Pr[R_H<3 m_2^2/m_1)&\leq &\Pr[R_{H\cap S_0}<3m_2^2/m_1)\nonumber\\
&\leq& \Pr[R_{H\cap S_0}<\frac{6}{11}\E[R_{H\cap S_0}))\nonumber\\
&\leq& \Pr[R_{H\cap S_0}-\E[R_{H\cap S_0})<-\frac{5}{11}\E[R_{H\cap S_0}))\nonumber\\
&\leq &e^{-O(\E[R_{H\cap S_0}))}\leq e^{-O(m_2^2/m_1)}.
\end{eqnarray}
Thus, the algorithm will output REJECT with probability $1-o(1)$ in this case.
\fi

\item[{\it Case} 2] Finally, suppose that $\sum_{i \in H\setminus S_0}|p_i-q_i|\geq \varepsilon/6$. Since $q_i>10/m_2$ for all $i \in H\setminus S_0$, it suffices to assume that $p_i\geq\frac{\varepsilon^{2/3}}{20m_2n^{1/3}}$. 
From Lemma \ref{lm:ZA}, letting $T = H \setminus S_0,$ we have that $\E[Z_T] \ge O(\eps^2 m_1^2m_2^2/36n),$ and    
$$\E[|Z_{ T}-\E[Z_{ T}]|^s]=O\left(\frac{n^{s/3} m_1^{s}m_2^{s+1}}{\varepsilon^{2s/3}}\right).$$ By Markov's inequality, 
\begin{eqnarray}
\Pr[Z_{ T}\leq C_\gamma m_1^{3/2}m_2/2]&\leq&\Pr[|Z_{ T}-\E[Z_{ T}]|\geq \Omega(m_1^{3/2}m_2)]\nonumber\\
&\leq &\widetilde O_s\left(\frac{n^{s/3} m_1^{s}m_2^{s+1}}{\varepsilon^{2s/3}m_1^{3s/2}m_2^s}\right)\nonumber\\
&\leq &\widetilde O_s\left(\frac{n^{s/3}m_2}{\varepsilon^{2s/3}m_1^{s/2}}\right).
\label{eq:mc}
\end{eqnarray}
If $m_2=\frac{n}{\sqrt m_1 \varepsilon^2}$ then (\ref{eq:mc}) becomes $\widetilde O_s\left(\frac{(n/\varepsilon^2)^{s/3+1}}{m_1^{s/2+1/2}}\right)$. Since $m_2\geq \Omega((n/\varepsilon^2)^{8/9})$, by  taking $s>5$, we can make the probability in (\ref{eq:mc}) $o(1)$. Similarly, if $m_1=n$ and $m_2=\sqrt n/\varepsilon^2$, then with $s=6$, (\ref{eq:mc}) becomes $\widetilde O_s\left(\frac{1}{\varepsilon^{8}\sqrt n}\right)=o(1)$ as $\varepsilon \geq n^{-\frac{1}{12}}$.  Together with the concentration of $Z_{S_0}$ from Chebyshev's inequality, we get that in this case, the $Z$ statistic check will fail and the algorithm will output REJECT with probability at least $0.99$ in this case.
\end{description} 
\end{proof}

\section{Lower Bound for $\ell_1$ Testing}\label{sec:lb}

In this section, we present lower bounds for the closeness testing problem under the $\ell_1$
norm using the machinery developed in Valiant \cite{pvaliant_stoc,pvaliant_thesis}. To this end, define the $(k_1, k_2)$-based moments $m(r, s)$ of a distribution pair $(p, q)$ as 
$k_1^rk_2^s\sum_{i=1}^n p_i^rq_i^s$. Valiant \cite[Theorem 4.6.9]{pvaliant_thesis} showed that if the distributions $p_1^+, p_2^+$ have probabilities at most $1/1000k_1$, and  $p_1^-, p_2^-$ have probabilities at most $1/1000k_2$, and 
\begin{equation}
\sum_{r+s> 1}\frac{|m^+(r, s)-m^-(r,s)|}{\sqrt{1+\max\{m^+(r,s), m^-(r,s)\}}}<\frac{1}{1000}.
\label{eq:lowerbound_condition}
\end{equation}
then the distribution pair $(p_1^+, p_2^+)$ cannot be distinguished with probability 13/24 from $(p_1^- , p_2^-)$ by a tester that takes $\dPois(k_1)$ samples from $(p_1^+, p_2^+)$ and $\dPois(k_2)$ samples from $(p_1^-, p_2^-)$.

Using this we prove the following proposition:

\begin{proposition}Let $ n^{2/3}/\varepsilon^{4/3}\leq m_1\leq n$. Then there exists distributions $p$ and $q$ such that given $\Theta(m_1)$ samples from $p$ requires $\Omega(\frac{n}{\sqrt m_1\varepsilon^2})$ samples from $q$ to distinguish between $p = q$ and $||p-q||_1 \geq \varepsilon$ with high probability.
\label{ppn:testing_lb}
\end{proposition}

\begin{proof}
Fix $\delta=1/4$. Let $b=1/m_1$ and $a=C/n$, where $C$ is an appropriately chosen constant.  Let $A$, $B$, and $C$ be disjoint subsets of size $(1-\delta)/b$, $1/a$, $1/a$, respectively.
Consider two distributions
$$p=b\boldsymbol 1_A+\delta a\boldsymbol 1_B,$$
and 
$$q=b\boldsymbol 1_A+\delta a(1+\varepsilon z)\boldsymbol 1_B,$$
where  $z$ is 1 or -1 depending on whether the index is even or odd (this is done so that $\sum_{i=1}^n q_i=1$). Then clearly $||p-q||_1=\delta\varepsilon=\varepsilon/4$.

Define $k_1=cm_1$ and $k_2=c\varepsilon^{-2}n/\sqrt m_1$, where $c$ is a sufficiently small constant.\ 
Then $||p||_\infty=b\leq \frac{1}{1000k_1}$ and $||p||_\infty=b\leq \frac{1}{1000k_2}$, whenever $m_1\geq n^{2/3}/\varepsilon^{4/3}$ and $b \geq a$.

Let $(p, p]=(p_1^+, p_2^+)$ and $(p, q]=(p_1^{-}, p_2^{-})$ and computing the $(k_1, k_2)$-based moments gives:
$$m^+(r, s]=k_1^rk_2^s (1-\delta)b^{r+s-1}+k_1^rk_2^s\delta^{r+s} a^{r+s-1},$$ 
and
$$m^-(r, s]=k_1^rk_2^s (1-\delta)b^{r+s-1}+k_1^rk_2^s\delta^{r+s} a^{r+s-1}\left(\frac{(1+\varepsilon)^s+(1-\varepsilon)^s}{2}\right).$$ 

By Theorem 4.6.9 of Valiant \cite{pvaliant_thesis}, to show that $(k_1, k_2)$ samples are not enough, it suffices to have (\ref{eq:lowerbound_condition}). Observe,
$$\frac{|m^+(r, s)-m^-(r,s)|}{\sqrt{1+\max\{m^+(r,s), m^-(r,s)\}}}\leq \frac{k_1^rk_2^s\delta^{r+s} a^{r+s-1}\left(1-\frac{1}{2}((1+\varepsilon)^s+(1-\varepsilon)^s)\right)}{\sqrt{k_1^rk_2^s (1-\delta)b^{r+s-1}}}.$$
For any $s\geq 0$, define $h(\varepsilon, s]=1-\frac{(1+\varepsilon)^s+(1-\varepsilon)^s}{2}$
Observe that $h(\varepsilon, 1]=0$, and $|h(\varepsilon, s)|\leq 1$, for $s\ne 1$. Note that $m_1\geq n^{2/3}/\varepsilon^{4/3}$, implies that $\varepsilon\geq n^{-\frac{1}{4}}$. Therefore, for every fixed $r\geq 0$ and $s\ne 1$, 
$$h(\varepsilon , s)k_1^{\frac{r}{2}}k_2^{\frac{s}{2}}b^{-(r+s-1)/2}a^{r+s-1}\leq c^{\frac{r+s}{2}}\left(\frac{m_1}{n}\right)^r\left(\frac{m_1^{\frac{1}{2}}}{\varepsilon^2 n}\right)^{\frac{s}{2}-1}\leq c^{\frac{r+s}{2}}\left(\frac{m_1}{n}\right)^{r+\frac{s}{4}+\frac{1}{2}}<c^{\frac{r+s}{2}},$$
since $m_1\leq n$ by assumption. This shows (\ref{eq:lowerbound_condition}) if $c$ is chosen small enough.
\end{proof}

%Let $\alpha=2/3+\gamma$. Let $\beta<\alpha$ and $b=c n^{-\alpha}$, where $c$ is some constant, and $a=4/n$. Consider two distributions
%$$p=b\boldsymbol 1_A+\varepsilon a\boldsymbol 1_B, \text{~and~} q=b\boldsymbol 1_A+\varepsilon a\boldsymbol 1_C,$$
%where $A$, $B$, and $C$ are disjoint subsets of size $(1-\varepsilon)/b$, $1/a$, $1/a$, respectively. Then clearly $||p-q||_2=2\varepsilon$.
%
%Define $k_1=c_1 n^{\alpha}$ and $k_2=c_2 n^{\beta}$, with $\beta < \alpha$, where $c_1, c_2$ are sufficiently small constants.\ 
%Then $||p||_\infty=b\leq \frac{1}{k_2}$ and $||p||_\infty=b\leq \frac{1}{k_1}$.
%
%Let $(p, p]=(p_1^+, p_2^+)$ and $(p, q]=(p_1^{-}, p_2^{-})$ and computing the $(k_1, k_2)$-based moments gives:
%
%$$m^+(r, s]=k_1^rk_2^s (1-\varepsilon)b^{r+s-1}+k_1^rk_2^s\varepsilon^{r+s} a^{r+s-1}, \text{~and~} m^-(r, s]=k_1^rk_2^s (1-\varepsilon)b^{r+s-1}.$$ Therefore,
%$$\frac{|m^+(r, s)-m^-(r,s)|}{\sqrt{1+\max\{m^+(r,s), m^-(r,s)\}}}\leq \frac{k_1^rk_2^s\varepsilon^{r+s} a^{r+s-1}}{\sqrt{k_1^rk_2^s (1-\varepsilon)b^{r+s-1}}}.$$
%By Theorem 4.6.9 of Valiant \cite{pvaliant_thesis}, to show that $(k_1, k_2)$ samples are not enough, it suffices to have
%$$\sum_{r+s> 1}\frac{|m^+(r, s)-m^-(r,s)|}{\sqrt{1+\max\{m^+(r,s), m^-(r,s)\}}}<\frac{1}{1000}.$$
%Observe that the exponent of $n$ in a general term of the above sum is $(r+s-1)-\alpha r /2-\beta s/2-\alpha(r+s-1)/2$, which at $r=s=1$ simplifies to $1-\alpha-\beta/2$. The sum converges whenever $1-\alpha-\beta/2>0$, which means $\beta=2(1-\alpha]=2/3-2\gamma$.

The optimality of the $\ell_1$ tester, establishing the lower bound in Theorem \ref{th:sample_size}, follows from the above proposition together with the lower bound of $\sqrt n/\varepsilon^2$ for testing uniformity given in Paninski \cite{paninsky}.

\end{document}